\documentclass{article}

\usepackage{microtype}
\usepackage{xcolor}
\usepackage{graphicx}
\usepackage{subcaption}
\usepackage{booktabs}
\usepackage{wrapfig}
\usepackage{paralist}
\usepackage{enumitem}
\setlist{nosep}
\usepackage{hyperref}

\usepackage[accepted]{icml2026}

\usepackage{amsmath}
\usepackage{amssymb}
\usepackage{mathtools}
\usepackage{amsthm}

\newtheorem{theorem}{Theorem}

\theoremstyle{remark}
\newtheorem{remark}{Remark}

\usepackage[capitalize,noabbrev]{cleveref}

\usepackage{algorithm}
\usepackage{algorithmic}
\usepackage{listings}
\newcommand{\dt}{\mathbf{d}_t}
\newcommand{\rt}{r_t}
\newcommand{\norm}[1]{\left\|#1\right\|}

\icmltitlerunning{Direction-Conditioned Policies for Online Goal-Conditioned RL}

\begin{document}

\twocolumn[
  \icmltitle{Direction-Conditioned Policies via Compositional Subgoal Scoring\\for Online Goal-Conditioned Reinforcement Learning}

  \icmlsetsymbol{equal}{*}

  \begin{icmlauthorlist}
    \icmlauthor{S K Swaminathan}{cse}
    \icmlauthor{Damiya Gondha}{cse,equal}
    \icmlauthor{Theyanesh Eswaramoorthy Rajahkrishnan}{me,equal}
    \icmlauthor{Aritra Hazra}{cse}
  \end{icmlauthorlist}

  \icmlaffiliation{cse}{Department of Computer Science and Engineering, Indian Institute of Technology Kharagpur, Kharagpur, India}
  \icmlaffiliation{me}{Department of Mechanical Engineering, Indian Institute of Technology Kharagpur, Kharagpur, India}

  \icmlcorrespondingauthor{S K Swaminathan}{swami2004@kgpian.iitkgp.ac.in}

  \icmlkeywords{Self-supervised RL, Goal-conditioned RL, Contrastive RL, Direction Conditioning, Online RL, Compositional RL}

  \vskip 0.3in
]

\printAffiliationsAndNotice{\icmlEqualContribution}

\begin{abstract}
Hamilton--Jacobi--Bellman theory implies that the optimal goal-conditioned action depends on the goal only through the gradient of the goal-reaching distance at the current state, yet standard online GCRL still conditions the actor on the raw goal -- a signal that is geometrically uninformative when the goal is far from the data distribution. We propose \emph{Direction-Conditioned Policies} (DCP), a fully online method that decomposes goal-reaching into two components sharing one InfoNCE representation $\psi$: a subgoal-scoring step that selects a visited state $z_t$ aligned with the final goal $g$ in $\psi_g$, and a direction-conditioned actor that consumes the unit direction $\dt$ and magnitude $\rt$ from $\psi(s_t)$ to $\psi(z_t)$. The two components train jointly, factor cleanly at deployment (subgoal scoring is removed, while direction conditioning remains with $g$ in place of $z_t$), and admit independent modification at the same $(\dt,\rt)$ interface. We prove three results. First, direction sufficiency under HJB: the optimal action under control-affine dynamics depends on the goal only through the value gradient. Second, a quantitative bound showing that, under mild conditions on the learned representation and assuming the scoring rule returns an on-path $z_t$, the actor's conditioning input at training and at deployment coincide up to representation error and geodesic slack. Third, a controllable-subspace characterization of when directional conditioning fails. Across nine environments spanning navigation and manipulation, DCP improves over Contrastive RL on most final metrics, with the largest gains on manipulation and obstacle-interaction tasks; a qualitative analysis of the learned $\psi$-distance landscape shows the contrastive representation behaves as an online quasimetric encoding environment topology, and the single failure case (AntSoccer) localizes to a learned-gradient pathology that the theory anticipates.\footnote{Code and reproduction scripts: \url{https://github.com/ApRL-IITKGP/direction-conditioned-policy}}
\end{abstract}

\section{Introduction}

Simulation is a scalable alternative to teleoperation for robot manipulation, enabling millions of parallel rollouts with contact-rich dynamics and no human operator. The challenge is sparse rewards: when goals are distant, the actor receives little direct signal. Standard Goal-Conditioned RL (GCRL) conditions the actor on the goal directly~\citep{Kaelbling1993LearningTA}. Contrastive RL (CRL)~\citep{Eysenbach2022ContrastiveLA} learns a representation $\psi$ via InfoNCE such that $\langle \psi(s,a), \psi(g) \rangle$ estimates log-reachability of $g$ (goal) from $(s,a)$. However, the actor still receives the raw goal coordinate as input. This signal is geometrically uninformative when $g$ is far from the data distribution: early in training $\psi_g(g)$ for a distant $g$ is poorly calibrated, so any direction computed from $\psi_g(g)$ is unstable. We argue that the appropriate conditioning input is not the goal coordinate itself but a \emph{direction} in the learned representation, and that this direction should be computed from a \emph{visited} state $z_t$ whose encoding the contrastive objective has already trained on.

Direction-Conditioned Policies (DCP) maintains a pool of visited states, selects $z_t$ from the pool by an inner-product scoring rule that maximizes alignment with $\psi_g(g)$, and conditions the actor on the unit direction $\dt$ from $\psi(s_t)$ to $\psi(z_t)$ together with a scalar magnitude $\rt$. Because $z_t$ is selected to be aligned with $g$ under $\psi_g$, the unit direction toward $z_t$ approximates the unit direction toward $g$ in $\psi$-space. At deployment, after $\psi$ has been trained, the same unit direction is computed with $g$ in place of $z_t$ and the pool is no longer needed. We also evaluate a simpler variant, \emph{Self-Subgoal Conditioning} (SSGC), which uses the same scoring but conditions on the raw subgoal coordinate $z_t$ instead of the direction; this ablation isolates the contribution of directional abstraction. No teleoperation, demonstrations, or offline data are required.
\begin{figure}[t]
  \centering
  \includegraphics[width=\linewidth]{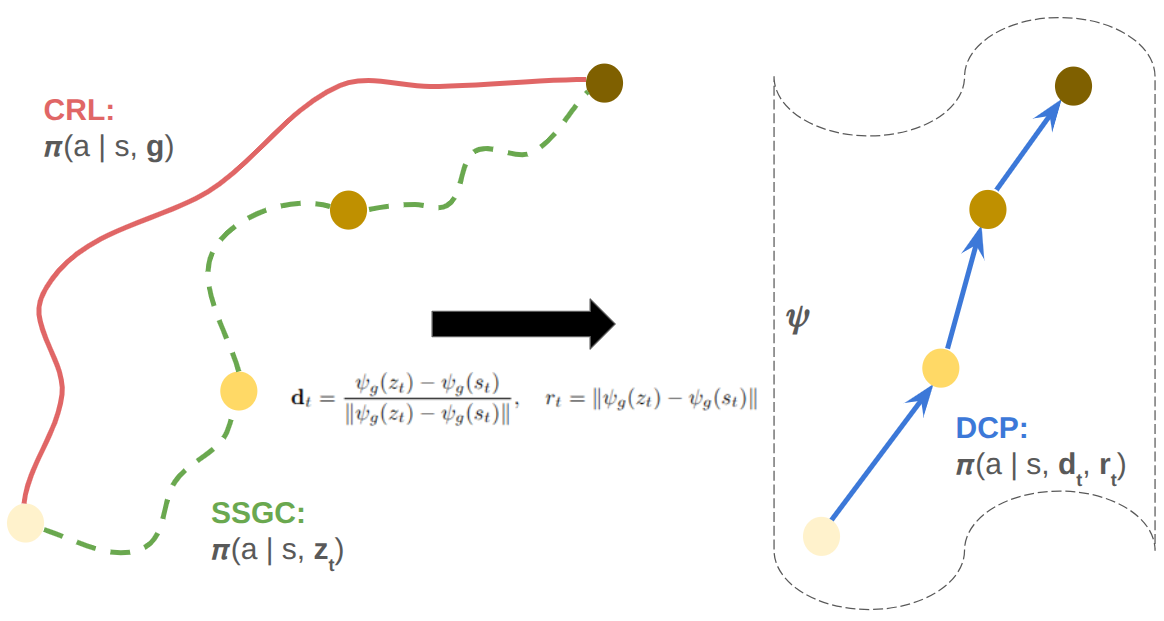}
  \caption{Conditioning Strategies. CRL: final goal $g$. SSGC: subgoal $z_t$. DCP: direction $\dt$ and magnitude $\rt$ in $\psi$-space.}
  \label{fig:method-overview}
\end{figure}

DCP decomposes goal-reaching into two pieces sharing one InfoNCE representation: scoring selects a visited state aligned with the goal in $\psi$-space, and a direction-conditioned actor consumes the unit direction from the current state toward that selected state. The two pieces train jointly but factor cleanly at deployment, where the pool is no longer needed and the actor is fed the unit direction toward $g$ directly. We give this design a principled foundation in Section~\ref{sec:theory}: the Hamilton--Jacobi--Bellman (HJB) equation identifies the unit gradient of the goal-reaching distance as the minimal sufficient statistic for the optimal action under control-affine dynamics. DCP is the online realization of this fact through CRL's InfoNCE-learned quasimetric. In heuristic-search vocabulary, $\psi$ is a learned consistent quasimetric and the actor performs greedy descent on it in continuous time.

The primary contributions of this work are as follows.
\begin{compactitem}
    \item We propose a direction-conditioned actor for online GCRL using $[\dt, \rt]$ in $\psi$-space.
    \item We introduce a subgoal scoring rule that selects visited states aligning with the goal in $\psi_g$.
    \item We establish three theoretical results: direction sufficiency under HJB, planning invariance at the conditioning interface, and a controllable-subspace failure characterization.
    \item We evaluate the proposed method across nine environments, showing five-seed gains over Contrastive RL on most final metrics, with AntSoccer failure consistent with the regime anticipated by theory.
\end{compactitem}

\section{Related Work}

DCP connects several lines of prior work: online contrastive goal-reaching, HJB/quasimetric formulations of control, and compositional subgoal methods. The contribution is an online actor-conditioning scheme that consumes the HJB-direction-sufficient statistic, enforces planning invariance at the conditioning interface, and remains orthogonal to representation-improvement work within Contrastive RL.

\paragraph{HJB-based and Quasimetric Methods for Goal-reaching.} The Hamilton--Jacobi--Bellman equation~\citep{BardiCD1997,Bertsekas2017DP} characterizes the optimal action via the gradient of the cost-to-go (Theorem~\ref{thm:hjb}). \citet{Wang2023OptimalGR} learn an explicit quasimetric for offline goal-reaching (QRL); \citet{Giammarino2025Eikonal} impose the Eikonal PDE constraint on a hierarchical quasimetric (Eik-QRL) and explicitly use the HJB form for value learning. Both target value learning rather than actor conditioning. \citet{Myers2024LearningTD} show that contrastive successor features can be reparameterized to satisfy the triangle inequality, providing the consistency property (in the heuristic-search sense). DCP is, to our knowledge, the first online actor-conditioning method that consumes the gradient of this learned consistent quasimetric; the orthogonality to prior value-learning work in this area is by construction.

\paragraph{Dual Goal Representations and Planning Invariance.} \citet{Park2025Dual} prove that the dual representation $\phi^\vee(g) := s\mapsto d^*(s,g)$ is sufficient to recover the optimal goal-conditioned policy and is invariant to exogenous goal-observation noise; the DCP conditioning vector $\dt$ is the local gradient of this dual representation at the current state, inheriting noise invariance under the same identification. \citet{Myers2025Horizon} define planning invariance -- a policy taking the same action whether conditioned on an on-path waypoint or the final goal -- and prove that some learned representations have it as an emergent property, with implications for horizon generalization. Theorem~\ref{thm:plan-inv} establishes a related statement \emph{at the actor's interface}: when the learned $d_\psi$ approximates $d^*$ uniformly to within $\delta$ and the scoring rule returns an $\varepsilon$-on-path $z$, the actor's conditioning input at training (toward $z$) and at deployment (toward $g$) coincide up to $O(\sqrt\delta + \sqrt\varepsilon)$. Planning invariance is thereby inherited at the conditioning interface from the representation's approximation properties, without an explicit hierarchy. HILP~\citep{Park2024FoundationPW} similarly conditions on Hilbert-metric directions but requires offline metric supervision and a symmetric metric; CRL's quasimetric is non-symmetric and DCP recovers the directional conditioning online via InfoNCE.

\paragraph{Online Contrastive RL.} CRL~\citep{Eysenbach2022ContrastiveLA} is our baseline: an InfoNCE critic paired with SAC. \citet{Bortkiewicz2024AcceleratingGR} establish CRL as the reference online GCRL method and provide the JAX infrastructure we build on. Scaled CRL~\citep{Wang20251000LN} improves via depth scaling (1000-layer networks); E-CRL~\citep{Tangri2025EquivariantGC} adds equivariance; both improve the \emph{representation}, while DCP improves the \emph{actor's conditioning interface} given any representation. The two contributions are orthogonal and can be integrated. Our empirical comparison is therefore a controlled CRL-backbone interface study: critic, replay, optimizer, training budget, and implementation are held fixed while the actor-conditioning signal is changed. \citet{Liu2024ASG} show that online contrastive RL with a single fixed goal can elicit useful skills and exploration without any subgoal machinery, contrasting with our explicit subgoal scoring. \citet{Park2025HorizonRM} argue that reducing the effective planning horizon is a key lever for scaling RL; we view this as complementary, since we modify the actor's conditioning input rather than the training-time horizon itself.

\paragraph{Embedding-difference Conditioning.} Plan Arithmetic / Compositional Plan Vectors (CPV)~\citep{Devin2019PlanAC} conditions a policy on differences in learned plan embeddings, enabling arithmetic composition over task or trajectory-level representations. DCP shares the high-level idea that a difference vector in representation space can be more useful than raw goal coordinates. The operating regime is different: CPV is one-shot imitation over reference/partial trajectories for explicit skill compositions, whereas DCP is online GCRL with final goals and no demonstrations. Applying CPV directly would therefore require changing the evaluation setting by adding reference trajectories and explicit task-composition structure.

\paragraph{Subgoal and Landmark Planning.} DCP's visited-state scorer is related to subgoal and landmark methods such as L3P / World Model as a Graph~\citep{Zhang2020WorldMA}, which learns latent landmarks and plans over a reachability graph. The shared intuition is that intermediate states can make long-horizon sparse goal-reaching easier. The distinction is architectural: L3P couples a goal-conditioned Q-learning controller with persistent landmark nodes and deployment-time graph search. DCP uses subgoals only as a training-time scaffold; at deployment there is no landmark graph, planner, or waypoint pool. Thus L3P is related landmark-planning work, while the experiments here isolate the effect of changing the actor-conditioning interface.

\paragraph{Compositional and Hierarchical RL.} A line of recent work decomposes long-horizon goal-reaching into reusable sub-units. \citet{Yalcinkaya2024Compositional} encode tasks as automata embeddings; \citet{Liang2024SkillDiffuser} learn interpretable skill abstractions via diffusion; \citet{Bakirtzis2024Reduce} formalize compositional RL through a categorical lens. Most of these methods rely on an explicit two-level architecture (separate high- and low-level policies, often trained at different rates) or external task structure. DCP is \emph{flat}: a single direction-conditioned actor plus a scoring rule, sharing one $\psi$ trained jointly with InfoNCE. The two components share one interface ($\dt, \rt$ in $\psi$-space) and can be modified independently behind it -- scoring selects $z_t$, direction conditioning consumes the resulting unit vector -- without explicit hierarchical structure.

\paragraph{Offline Goal-conditioned RL.} HIQL~\citep{Park2023HIQLOG}, \citet{Myers2025OfflineGR}, and \citet{Zheng2025MultistepQL} learn from pre-collected datasets, exploiting offline metric supervision DCP does not require. \citet{Park2024OGBenchBO} introduce OGBench for benchmarking such methods. We do not run head-to-head comparisons; the operating regime (online with no demonstrations) differs.

\section{Proposed Method}
\label{sec:method}

\subsection{Contrastive Representation}
Following CRL~\cite{Eysenbach2022ContrastiveLA}, we train a state-action encoder $\psi_{sa}$ and goal encoder $\psi_g$ via InfoNCE:
\begin{equation}
\mathcal{L} = -\mathbb{E}\!\left[\log \frac{e^{\langle\psi_{sa}(s,a),\, \psi_g(g)\rangle}}{\sum_{g'} e^{\langle\psi_{sa}(s,a),\, \psi_g(g')\rangle}}\right]
\label{eq:infonce}
\end{equation}
After training, $\langle \psi_{sa}(s,a), \psi_g(g) \rangle$ is large when $(s,a)$ tends to lead to $g$ and small otherwise.

\subsection{Subgoal Pool and Scoring (SSGC \& DCP)}
Both methods maintain a fixed-size pool $\mathcal{P}$ of recently visited states, updated online. At each step a subgoal $z_t \in \mathcal{P}$ is selected by:
\begin{equation}
z_t = \arg\max_{z \in \mathcal{P}}\; \langle \psi_g(z), \psi_g(g) \rangle
\label{eq:score}
\end{equation}
The dot product selects pool states whose $\psi_g$-encoding is most similar to $\psi_g(g)$ under the InfoNCE inner product. Conceptually, these are states the contrastive objective has placed in the same region of $\psi$-space as the goal; the resulting unit direction toward $z_t$ is therefore close to the unit direction toward $g$ in $\psi$-space, while $z_t$ itself was actually visited and so is on the support of the goal encoder's training distribution.

\textbf{SSGC} conditions the actor on $[s_t, z_t]$: a concrete, achievable waypoint instead of the distant goal. \textbf{DCP} instead computes
\begin{align}
\dt &= \frac{\psi_g(z_t) - \psi_g(s_t)}{\norm{\psi_g(z_t) - \psi_g(s_t)}}, \quad
\rt = \norm{\psi_g(z_t) - \psi_g(s_t)}
\end{align}
and conditioning the actor on $[s_t, \dt, \rt]$. Here, $\dt \in \mathbb{S}^{d-1}$ is a unit vector in representation space pointing from $s_t$ toward $z_t$, and $\rt$ encodes the magnitude of the conditioning input in $\psi$-space. Algorithm~\ref{alg:dcp} lists the rollout step and the actor and critic losses in pseudocode.

Both methods are trained end-to-end with SAC~\cite{Haarnoja2018SoftAO} alongside the contrastive objective (Eq.~\ref{eq:infonce}), with no changes to the critic or replay buffer. The actor objective remains within the same CRL/JaxGCRL training stack; DCP changes the conditioning object passed to the actor and the positive target induced by the selected waypoint, not the critic, data collection, optimizer, or schedule. During \emph{deployment}, the subgoal pool is not needed since $\dt$ is computed directly as $(\psi_g(g) - \psi_g(s_t)) / \norm{\psi_g(g) - \psi_g(s_t)}$; therefore DCP is deployed as a standard goal-conditioned policy. $\psi_g$ is the same goal encoder CRL already trains and DCP incurs no extra training cost or deployment complexity.

\begin{algorithm}[t]
\caption{DCP: rollout step, subgoal scoring, and losses.}
\label{alg:dcp}
\begin{algorithmic}[1]
\REQUIRE goal $g$, visited-state pool $\mathcal{P}$, encoders $\psi_{sa}, \psi_g$, policy $\pi$
\STATE \textbf{Function} \textsc{Direction}$(\psi_z, \psi_s)$
\STATE \quad $\dt \gets (\psi_z - \psi_s) / \|\psi_z - \psi_s\|$;\quad $\rt \gets \|\psi_z - \psi_s\|$
\STATE \quad \textbf{return} $(\dt, \rt)$
\STATE
\STATE \textbf{Function} \textsc{SelectSubgoal}$(g, \mathcal{P})$ \COMMENT{every 25 env-steps, 32 of 512}
\STATE \quad $\mathcal{C} \gets \textsc{SampleWithoutReplacement}(\mathcal{P}, 32)$
\STATE \quad \textbf{return} $\arg\max_{z \in \mathcal{C}} \langle \psi_g(z),\, \psi_g(g)\rangle$
\STATE
\STATE \textbf{Function} \textsc{RolloutStep}$(s, z)$ \COMMENT{$z$ stored in buffer goal slot}
\STATE \quad $(\dt, \rt) \gets \textsc{Direction}(\psi_g(z),\, \psi_g(s[\text{goal dims}]))$
\STATE \quad \textbf{return} $\pi(s, \dt, \rt)$
\STATE
\STATE \textbf{Function} \textsc{CriticLoss}$(s, a, g')$ \COMMENT{backward InfoNCE over $(s,a)$}
\STATE \quad $\phi \gets \psi_{sa}(s,a)$;\quad $\psi \gets \psi_g(g')$
\STATE \quad $\ell_{ij} \gets -\|\phi_i - \psi_j\|$ \COMMENT{negative-L2 energy}
\STATE \quad \textbf{return} $-\,\mathrm{mean}_i\bigl[\ell_{ii} - \mathrm{logsumexp}_j\,\ell_{ji}\bigr]$
\STATE
\STATE \textbf{Function} \textsc{ActorLoss}$(s, z)$ \COMMENT{encoders \texttt{stop\_grad}; SAC entropy added}
\STATE \quad $\psi_z \gets \psi_g(z)$;\quad $\psi_s \gets \psi_g(s[\text{goal dims}])$
\STATE \quad $(\dt, \rt) \gets \textsc{Direction}(\psi_z, \psi_s)$;\quad $a \gets \pi.\textsc{Sample}(s, \dt, \rt)$
\STATE \quad \textbf{return} $\mathrm{mean}\,\|\psi_{sa}(s,a) - \psi_z\| \;+\; \alpha\,\log\pi(a|s,\dt,\rt)$
\end{algorithmic}
\end{algorithm}

\subsection{Rationale behind Scoring and Direction}
\label{sec:why-score-why-direction}

\paragraph{Why score subgoals from a visited pool?}
The conceptually simplest version of direction conditioning would feed the actor $\dt = (\psi_g(g)-\psi_g(s_t))/\norm{\psi_g(g)-\psi_g(s_t)}$ directly. We avoid this early in training, because the encoded goal direction is not yet informative: $\psi$ has not been shaped by the InfoNCE objective enough for $\psi_g(g)$ to be meaningfully separated from arbitrary visited states in $\psi$-space. In this regime $\langle \psi_g(g), \psi_{sa}(s,a)\rangle$ is close to chance and the unit vector $\dt$ is unstable. Conditioning on a miscalibrated direction is worse than conditioning on the raw goal coordinate (see SSGC and CRL ablations in Sec.~\ref{sec:results}).

Subgoal scoring mitigates this issue. Equation~\eqref{eq:score} selects $z_t$ as the visited state whose $\psi_g(z_t)$ is most aligned with $\psi_g(g)$ under the InfoNCE inner product, restricting attention to candidates the contrastive objective has already separated from the rest of the buffer. The chosen $z_t$ is therefore a state for which the direction $\psi_g(z_t)-\psi_g(s_t)$ is computed from \emph{two} encodings the contrastive objective has trained on, rather than one trained encoding ($\psi_g(s_t)$) and one poorly calibrated encoding ($\psi_g(g)$ for far-from-data $g$). At deployment, after InfoNCE training has shaped $\psi$, the same construction with $z_t$ replaced by $g$ becomes calibrated and the pool is no longer needed.

\paragraph{Why condition on a direction in $\psi$-space rather than raw $z_t$?}
SSGC conditions on $z_t$ directly. This can be effective during training but degrades at deployment, because the actor sees the distribution of subgoal coordinates during training and the distribution of \emph{final} goal coordinates at test time, and the two need not coincide  --  a train--test mismatch documented in our SSGC results. DCP avoids this by mapping both training-time and deployment-time conditioning into the same object: the unit direction $\dt$ in $\psi$-space. Because scoring selects $z_t$ to be aligned with $g$ in $\psi_g$, the unit direction toward $z_t$ approximates the unit direction toward $g$, and hence at deployment, we can compute the same unit direction with $g$ in place of $z_t$.

Theorem~\ref{thm:plan-inv} formalizes this as planning invariance at the conditioning interface. HILP~\citep{Park2024FoundationPW} learns a Hilbert metric offline and conditions on directions in that metric; DCP recovers an analogous directional conditioning \emph{online} via the InfoNCE quasimetric of CRL, without offline metric supervision. CRL's quasimetric is non-symmetric~\citep{Myers2024LearningTD}, which lets DCP express asymmetric goal-reaching geometry that a symmetric Hilbert embedding cannot.

\section{Theoretical Formulation: Direction Conditioning under HJB}
\label{sec:theory}

We work in continuous time on $\mathcal{S}\subseteq\mathbb{R}^n$, $\mathcal{A}\subseteq\mathbb{R}^m$, with deterministic locally-Lipschitz dynamics $\dot s = f(s,a)$ and lower-semicontinuous cost $c\ge 0$. The optimal goal-reaching cost-to-go is $d^*(s,g) := \inf_{a(\cdot)} \int_0^{T_g} c\,\mathrm dt$ subject to $s(0)=s$, $s(T_g)=g$, and we assume $d^*(\cdot,g)$ is $C^1$ on $\mathcal{S}\setminus\{g\}$ (the standard regularity hypothesis under which the HJB equation has classical solutions; the viscosity-solution extension is well-known~\citep{BardiCD1997}). The learned representation $\psi:\mathcal{S}\to\mathbb{R}^k$ induces a distance $d_\psi(s,g):=\|\psi(g)-\psi(s)\|$, and DCP feeds the actor the unit direction $\dt = -\nabla_s d_\psi/\|\nabla_s d_\psi\|$ and magnitude $\rt = d_\psi$. The implementation uses the equivalent $\psi$-space form $\dt = (\psi(g)-\psi(s))/\|\cdot\|$; the two are related by the pullback $\nabla_s d_\psi = -J_\psi(s)^\top\widehat{\psi(g)-\psi(s)}$ and encode the same information when $J_\psi$ has full rank. Full assumptions and notation are in Appendix~\ref{app:theory-setup}.

Section~\ref{sec:theory} states the three theorems that justify DCP's design; full proofs and remarks are deferred to Appendix~\ref{app:theory}.

\begin{theorem}[{\bf Direction Sufficiency under HJB}]
\label{thm:hjb}
Under the regularity above, the optimal goal-conditioned feedback law $a^*(s,g)$ depends on $g$ only through $\nabla_s d^*(s,g)$: there is a function $\bar a:\mathcal{S}\times\mathbb{R}^n\to\mathcal{A}$ such that $a^*(s,g)=\bar a(s,\nabla_s d^*(s,g))$. If, in addition, the dynamics are control-affine with $f(s,a)=f_0(s)+G(s)a$ and $c(s,a)=c_0(s)+\tfrac12 a^\top R(s) a$, $R(s)\succ 0$, then $\bar a$ is \emph{linear} in its second argument and $a^*(s,g)=-R(s)^{-1}G(s)^\top\nabla_s d^*(s,g)$. The unit direction $\widehat{\nabla_s d^*}$ sets the orientation of $a^*$ and the magnitude $\|\nabla_s d^*\|$ sets its scale.
\end{theorem}

\emph{Intuition:} the HJB equation $\inf_a [c(s,a)+\nabla_s d^*\cdot f(s,a)]=0$ couples $g$ to the optimization only through the gradient term, and minimizing over $a$ yields a feedback law that is a function of $(s,\nabla_s d^*)$ alone. DCP's conditioning interface $(\dt,\rt)$ is precisely this minimal sufficient statistic with $d^*$ replaced by the learned $d_\psi$. The full proof, the connection to Pontryagin's principle, and the viscosity-solution extension are in Appendix~\ref{app:thm1-proof}.

\begin{theorem}[{\bf Planning Invariance at Conditioning Interface}]
\label{thm:plan-inv}
Assume \emph{driftless} control-affine dynamics $f(s,a)=G(s)a$ with quadratic cost (B0); a state $z$ that is \emph{$\varepsilon$-on-path} from $s$ to $g$ in the sense $d^*(s,g)\le d^*(s,z)+d^*(z,g)\le d^*(s,g)+\varepsilon$ (B1); a learned $d_\psi$ approximating $d^*$ uniformly to within $\delta$ with $L$-Lipschitz gradients in both arguments (B2); and $\|\nabla_s d^*(s,g)\|\ge m>0$ (B3). Let $\dt^{(z)}=\widehat{-\nabla_s d_\psi(s,z)}$ and $\dt^{(g)}=\widehat{-\nabla_s d_\psi(s,g)}$. Then
\begin{equation}
\bigl\|\dt^{(z)} - \dt^{(g)}\bigr\| \;\le\; \tfrac{8\sqrt{2L\delta}}{m}\;+\;\tfrac{2\sqrt{2L\varepsilon}}{m}.
\label{eq:plan-inv-bound}
\end{equation}
\end{theorem}

{\bf Intuition:} under (B0) the optimal trajectories of $d^*(\cdot,g)$ and $d^*(\cdot,z)$ share a common geodesic tangent at $s$ when $z$ is on-path, so the unit gradients agree to $O(\sqrt\varepsilon)$; the Glaeser-type $L^\infty\!\to\!W^{1,\infty}$ inequality converts $\delta$-uniform value approximation into $O(\sqrt\delta)$ gradient perturbation, which the unit-vector stability lemma (using B3) carries to the unit direction. Consequently, for a $K_\pi^d$-Lipschitz policy, the actor's training-time and deployment-time outputs agree up to $O(\sqrt\delta+\sqrt\varepsilon)$. The statement is conditional on the on-path assumption: the inner-product scorer is not claimed to guarantee (B1) for arbitrary learned representations, and Appendix~\ref{app:thm2-proof} states this explicitly. The driftless hypothesis is used only for the geodesic-colinearity step; Appendix~\ref{rem:drift} details which parts remain valid with drift. Full proof, drift-extension discussion, and the connection to dual goal representations~\citep{Park2025Dual} and horizon generalisation~\citep{Myers2025Horizon} are in Appendix~\ref{app:thm2-proof}.

\begin{theorem}[{\bf Failure Mode: Uncontrollable Goal Subspace}]
\label{thm:fail}
Under control-affine $f(s,a)=f_0(s)+G(s)a$, define the controllable subspace $\mathcal{C}(s):=\mathrm{im}(G(s))$ and the controllable gradient component $\nabla_s^{\mathcal C}d^*:=\Pi_{\mathcal C(s)}\nabla_s d^*$. Then (i) $a^*$ depends on $\nabla_s d^*$ only through $\nabla_s^{\mathcal C}d^*$; (ii) if $\|\nabla_s^{\mathcal C}d^*(s,g)\|\le\rho\|\nabla_s d^*(s,g)\|$ for $\rho\ll 1$, then under quadratic cost $\|a^*(s,g)\|\le\rho\,\|R^{-1}\|_{\mathrm{op}}\|G\|_{\mathrm{op}}\|\nabla_s d^*\|$ and the directional gradient signal $\dt$ is uninformative; (iii) on a state $z$ for which $\|\nabla_s^{\mathcal C}d^*(s,z)\|\ge\Theta(\|\nabla_s d^*(s,z)\|)$, $\dt(s,z)$ is informative and the actor can make progress toward $z$.
\end{theorem}

{\bf Intuition:} $G(s)^\top$ annihilates $\mathcal{C}(s)^\perp$, so only the controllable component of the gradient enters the action selection; when this component is small, every action with $G(s)a\perp\nabla_s^{\mathcal C}d^*$ is approximately optimal and the directional signal carries no usable information. The full proof and the AntSoccer instantiation (where the failure is on the \emph{learned} gradient $\nabla_s d_\psi$ rather than $\nabla_s d^*$, due to a narrow visited pool that does not expose ball-relevant directions to InfoNCE) are in Appendix~\ref{app:thm3-proof} and~\ref{app:antsoccer-theory}.

\paragraph{Compositional Reading.} In heuristic-search vocabulary, $\psi$ is a learned consistent quasimetric (consistency from the contrastive-successor-feature triangle inequality of \citet{Myers2024LearningTD}); $\dt$ is the negative gradient of this quasimetric; the actor performs greedy best-first descent in continuous time. The waypoint pool plays the role of an open list during training; at deployment, once $\psi$ has been trained, the list collapses to $\{g\}$ and the policy becomes a standard goal-conditioned controller. We frame this as an analogy motivating the algorithm rather than a formal claim: full A$^*$-style admissibility ($d_\psi\le d^*$ pointwise) is not implied by InfoNCE training and is not required; the formal claims are Theorems~\ref{thm:hjb}--\ref{thm:fail}.

\begin{figure*}[t]
  \centering
  \includegraphics[width=\linewidth]{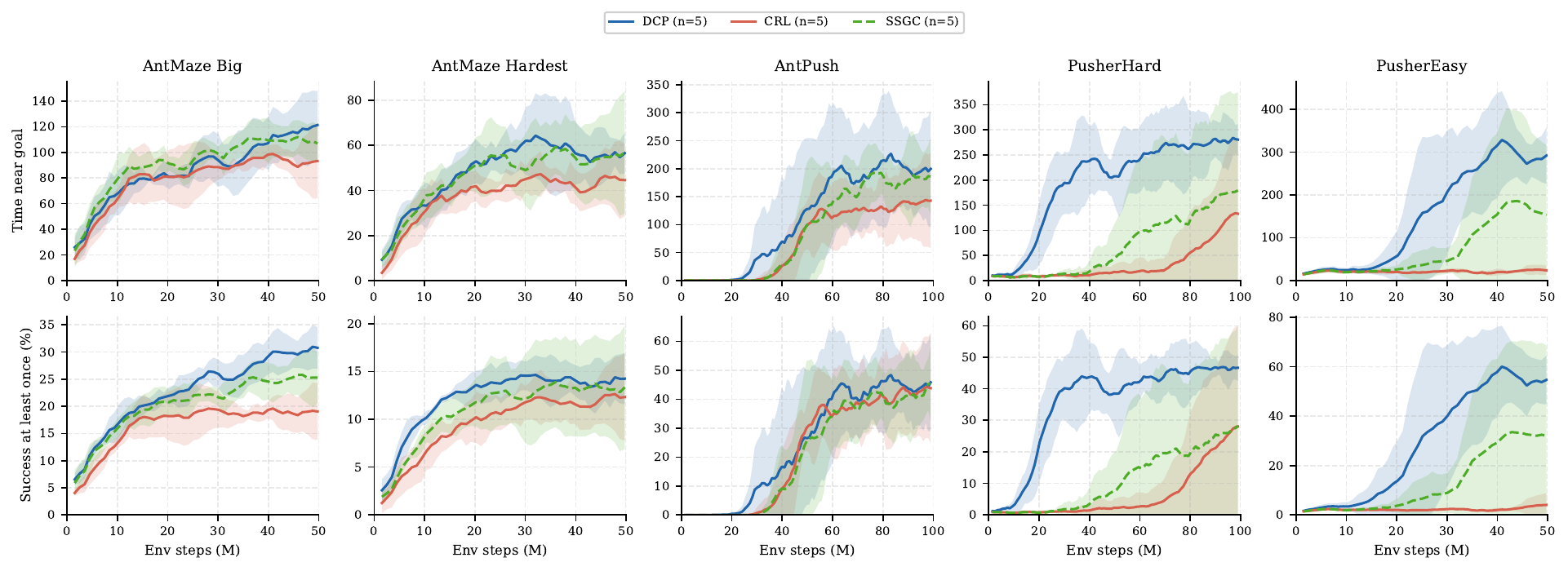}\\[2pt]
  \includegraphics[width=0.85\linewidth]{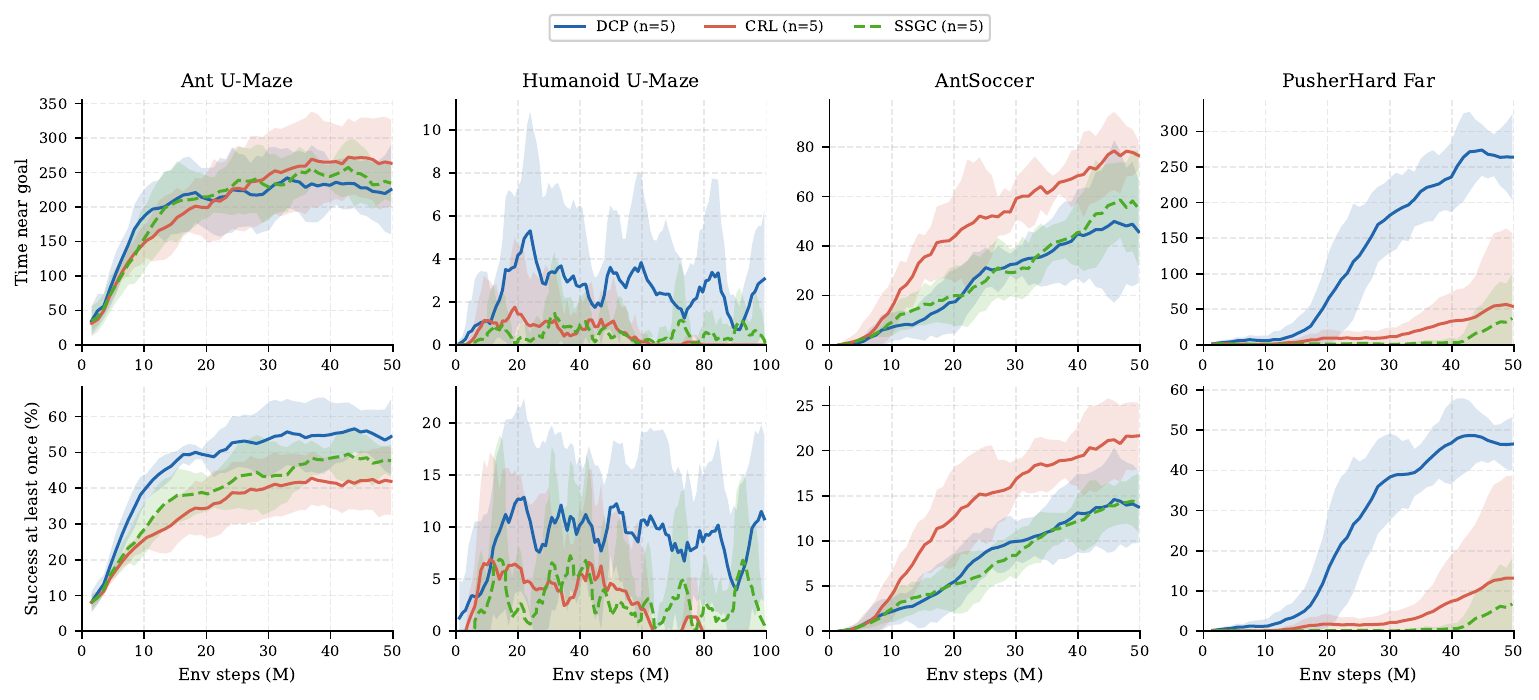}
  \caption{Learning curves for time near goal (top of each strip) and success at least once (bottom) across all nine environments. Lines show five-seed means and shaded bands show 95\% confidence intervals over seeds. AntPush, PusherHard, and Humanoid U-Maze use 100M-step budgets; the remaining environments use 50M-step budgets. Final five-seed values are summarized in Table~\ref{tab:summary}, with checkpointed means and confidence intervals in Appendix~Tables~\ref{tab:appx-eval-episode-success-any}--\ref{tab:appx-eval-episode-success}.}
  \label{fig:main-results}
\end{figure*}

\section{Experimental Results}
\label{sec:experiments}

\begin{table}[t]
\centering
\caption{Final DCP vs.\ CRL comparison across nine environments (five seeds each).}
\label{tab:summary}
\scriptsize
\setlength{\tabcolsep}{4pt}
\renewcommand{\arraystretch}{0.92}
\begin{tabular}{lccc}
\toprule
Environment & Task & Time Near Goal & Success $\geq$1 \\
\midrule
AntMaze Big      & Navigation     & $\uparrow$ DCP & $\uparrow$ DCP \\
AntMaze Hardest  & Navigation     & $\uparrow$ DCP & $\uparrow$ DCP \\
Ant U-Maze       & Navigation     & $\downarrow$ CRL & $\uparrow$ DCP \\
Humanoid U-Maze  & Navigation     & $\uparrow$ DCP & $\uparrow$ DCP \\
AntPush          & Nav + Manip    & $\uparrow$ DCP & $\approx$ \\
PusherEasy       & Manipulation   & $\uparrow$ DCP & $\uparrow$ DCP \\
PusherHard       & Manipulation   & $\uparrow$ DCP & $\uparrow$ DCP \\
PusherHard Far   & Manipulation   & $\uparrow$ DCP & $\uparrow$ DCP \\
AntSoccer        & Manipulation   & $\downarrow$ CRL & $\downarrow$ CRL \\
\bottomrule
\end{tabular}
\caption*{$\uparrow$/$\downarrow$ mark the higher final mean; $\approx$ marks a small or uncertainty-overlapped difference. Exact checkpointed means and 95\% confidence intervals are in Appendix~Tables~\ref{tab:appx-eval-episode-success-any}--\ref{tab:appx-eval-episode-success}. This table is an effect-size summary, not a significance test.}
\end{table}

\subsection{Setup and Environment Details}
We evaluate across nine environment configurations using Brax~\cite{Freeman2021BraxA}. Unless noted, runs use 512 parallel environments, 50M env steps, episode length 1000, unroll length 62, batch size 256, replay buffer size 10\,000, discount $\gamma{=}0.99$, and learning rate $3{\times}10^{-4}$.

AntPush, PusherHard, and Humanoid U-Maze are trained for 100M steps; Humanoid U-Maze uses 128 parallel environments and batch size 128 due to compute constraints. The subgoal pool holds 512 states; 32 candidates are scored every 25 steps. We use backward InfoNCE as the contrastive loss and negative L2 norm as the energy function, following the official JaxGCRL run scripts~\cite{Bortkiewicz2024AcceleratingGR}; CRL numbers may therefore differ from those reported in the JaxGCRL paper, which uses symmetric InfoNCE.
Primary experiments use five seeds; Figure~\ref{fig:main-results} reports learning curves with mean and 95\% confidence intervals over seeds, and Appendix~Tables~\ref{tab:appx-eval-episode-success-any}--\ref{tab:appx-eval-episode-success} give checkpointed values including the final checkpoint. The baseline throughout the main comparison is CRL~\cite{Eysenbach2022ContrastiveLA} with identical architecture, hyperparameters, and training budget; the only difference is actor conditioning. We therefore interpret the results as a controlled interface comparison within one strong online GCRL backbone, rather than as an exhaustive benchmark over all goal-conditioned RL families.
\begin{compactitem}
 \item \emph{AntMaze Big / Hardest}: a quadruped navigates increasingly complex mazes to reach a sparse goal.
 \item \emph{Humanoid U-Maze}: a 17-DOF humanoid navigates a U-shaped maze; the same task structure as AntMaze but with a far higher-dimensional body (268 state dimensions vs.\ 29).
 \item \emph{PusherEasy / PusherHard}: a 7-DOF arm pushes a cylinder to a goal XY position; Easy and Hard differ in goal-distribution difficulty.
 \item \emph{AntPush}: a quadruped navigates a room with a movable box obstacle to reach a goal XY position encoding the ant's own position.
 \item \emph{AntSoccer}: a quadruped must kick a ball to a target position; the goal is defined on the ball's XY position.
\end{compactitem}

We report two metrics, both using the same goal-region threshold:
\begin{compactitem}
 \item \emph{Time near goal} (primary): average timesteps per episode within the goal region, measuring sustained proximity.
 \item \emph{Success at least once}: fraction of episodes entering the goal region at least once, measuring reaching ability.
\end{compactitem}

\subsection{Qualitative Evaluation: Quasimetric Landscape}
\label{sec:qualitative}

\begin{figure*}[t]
  \centering
  \begin{minipage}[c]{0.6\linewidth}
    \centering
    \includegraphics[width=\linewidth]{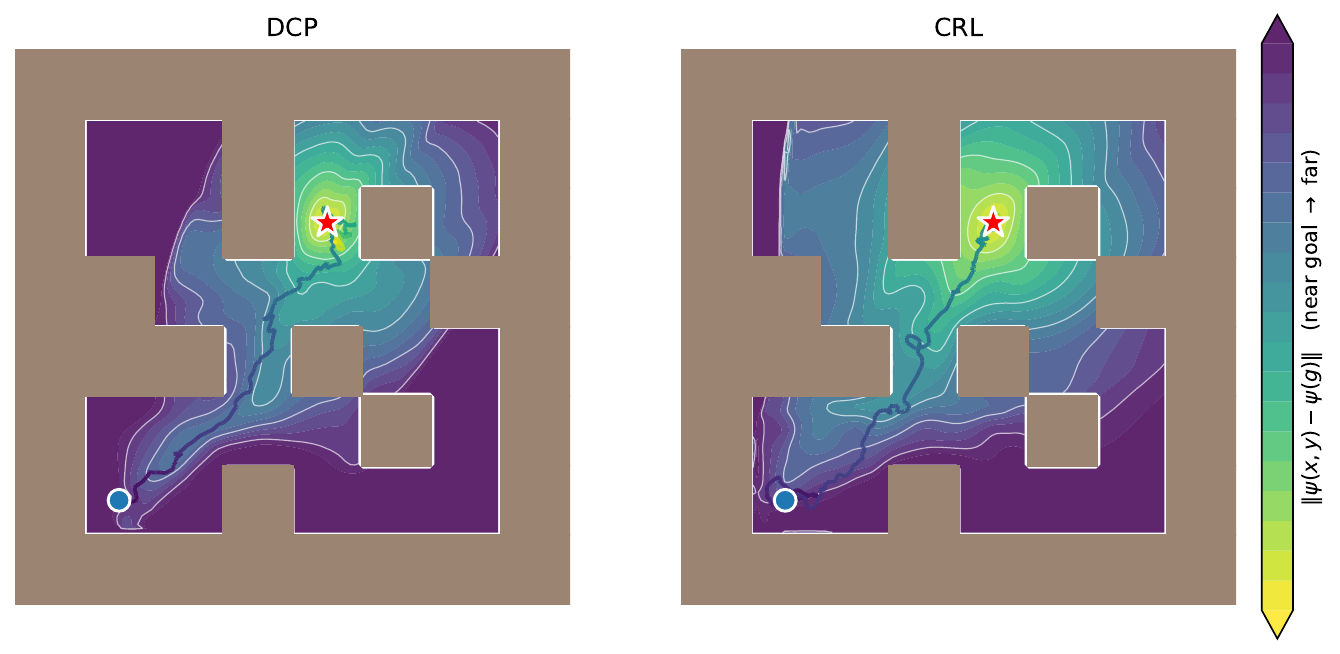}\\
    \footnotesize (a) $\psi$-distance landscape
  \end{minipage}\hfill
  \begin{minipage}[c]{0.36\linewidth}
    \centering
    \includegraphics[width=0.48\linewidth,trim=0 40 0 40,clip]{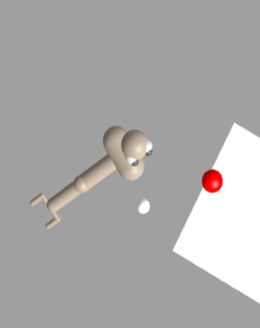}\hspace{2pt}%
    \includegraphics[width=0.48\linewidth,trim=0 40 0 40,clip]{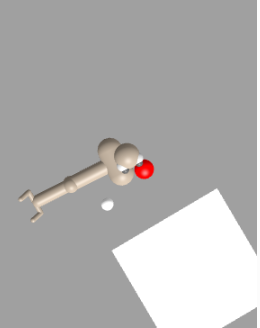}\\
    \footnotesize (b) PusherHard goal geometry
  \end{minipage}
  \caption{\textbf{(a)} Iso-distance contours of $\|\psi(x,y)-\psi(g)\|$ over AntMaze Big under DCP (left) and CRL (right) on a seed where both reach the goal; contours bend around walls in both panels (online quasimetric), and DCP's basin is sharper along the start--goal corridor. \textbf{(b)} PusherHard goal geometry (red: goal; white: cylinder). Left: PusherHard-Far (far arc, natural push). Right: default (near arc, requires repositioning behind the cylinder).}
  \label{fig:qualitative}
\end{figure*}

Figure~\ref{fig:qualitative}(a) renders $\|\psi(x,y) - \psi(g)\|$ over AntMaze Big at convergence: DCP's actor consumes the gradient of this field; CRL receives only the raw goal coordinate. We show a seed on which both methods reach the goal, so the comparison is about $\psi$ geometry rather than outcomes. Two structural claims are visible. First, contours bend around walls in both panels -- InfoNCE produces an online quasimetric encoding reachability through corridors~\cite{Wang2023OptimalGR,Myers2024LearningTD}, the same structure HILP~\cite{Park2024FoundationPW} obtains offline via explicit Hilbert-metric supervision. Second, DCP's basin is tight along the reachable corridor while CRL's is more diffuse, consistent with the joint-training coupling: DCP's actor reads $\psi$'s gradient at every step, so any $\psi$ drift away from a tight reachable basin directly hurts the policy gradient and is corrected during training; CRL's actor is $\psi$-independent and so does not impose this constraint. App.~Fig.~\ref{fig:psi-landscape-seed1} shows the same comparison on a seed where CRL fails; the diffuse-basin pattern is more pronounced there.

\subsection{Empirical Performance}
\label{sec:results}

\paragraph{Navigation -- AntMaze and HumanoidMaze.}
DCP improves reaching across the navigation tasks and is especially useful as the body dynamics become harder. On the AntMaze variants, DCP generally improves success-at-least-once over CRL, while time near goal is similar or better except for the final Ant U-Maze checkpoint. Humanoid U-Maze is difficult for all methods, but DCP is the only method with consistent nonzero progress, suggesting that directional conditioning remains useful even when the same maze structure is paired with a much higher-dimensional body. SSGC is competitive on the lower-dimensional mazes but does not scale as cleanly to Humanoid, where raw subgoal coordinates create a harder conditioning problem than directions in $\psi$-space.

\paragraph{Object Manipulation -- Pusher and AntPush.}
DCP has its clearest gains on tasks that require continuous corrective control around contact or obstacles. On the Pusher family, DCP improves both reaching and sustained proximity, while SSGC often lags CRL, indicating that raw subgoal conditioning is insufficient. AntPush also favors DCP on sustained occupancy: the time-near-goal gap opens during training and remains visibly higher for DCP across the later part of the 100M-step run, while success-at-least-once is comparable at the final checkpoint. Thus the AntPush gain is best read as a sustained-occupancy improvement rather than a first-contact improvement.

\paragraph{Goal Geometry Ablation.}
CRL performs worse on PusherHard than in the original JaxGCRL paper~\cite{Bortkiewicz2024AcceleratingGR}. We traced this to a goal distribution shift across JAX/Brax versions and found that our PRNG biases goals toward the near arc (Fig.~\ref{fig:qualitative}(b)), requiring repositioning behind the object. Scaled CRL~\cite{Wang20251000LN} documents similar discrepancies. To verify, we created \emph{PusherHard-Far}, constraining goals to the far arc. CRL improves relative to the harder near-arc setting, but DCP still leads on the final metrics, confirming robustness to goal geometry.

\paragraph{Failure Case -- AntSoccer.}\label{sec:antsoccer}
AntSoccer is a failure case for DCP, with CRL $>$ DCP $>$ SSGC throughout training. The mechanism is analysed in Section~\ref{app:antsoccer-theory}: the goal coordinate is the ball's position, and the visited pool early in training rarely contains states with displaced balls, so the InfoNCE objective receives little signal that distinguishes ball-position changes from ant-position changes. Empirically $\nabla_s d_\psi(s, g_{\text{ball}})$ lies in the controllable subspace (so the actor receives a non-trivial signal) but its projection onto the directions that actually reduce $d^*(s, g_{\text{ball}})$ is small. CRL is less exposed to this learned-gradient misalignment because its actor receives the raw goal coordinate, while the InfoNCE loss still provides a weaker but non-zero action-value signal through the SAC critic.

\subsection{Summary and Discussion}
DCP's advantage grows with task difficulty (Table~\ref{tab:summary}): the clearest gains appear on manipulation and obstacle-interaction tasks such as Pusher and AntPush, and on high-dimensional navigation such as Humanoid U-Maze. SSGC's underperformance on several manipulation tasks isolates the source: subgoal \emph{scoring} can help, but conditioning on raw subgoal coordinates creates a train--eval coordinate-distribution mismatch; DCP avoids this by mapping both training-time and deployment-time conditioning into the same $\psi$-space direction. DCP carries no deployment overhead beyond CRL: at test time $\dt$ is computed directly from $g$, with no pool or candidate scoring -- an advantage over methods requiring waypoint graphs or planner modules at test time.

\subsection{Saturation-Aware Zero-Shot Deployment}
\label{sec:pburst}
The deployment-time policy can be improved \emph{without retraining} by replacing the action with $\mathcal{U}[-1,1]$ for $L=10$ steps every $P=100$ steps. We hypothesize this helps DCP because DCP recomputes $\dt$ at every step from the current state -- a noise-induced perturbation produces a fresh corrective input -- whereas CRL reads only the static raw goal and has no analogous channel. Table~\ref{tab:pburst} reports paired deployment results across the same tasks. The intervention improves DCP more consistently than CRL, with the largest differentials on AntPush and the Pusher tasks. Per-environment action-saturation rates (Appendix~\ref{app:saturation}) confirm the precondition that deterministic deployment leaves headroom but do not predict the magnitude of $\Delta$; the CRL control indicates that the effect is tied to DCP's direction recomputation rather than to generic SAC action saturation.

\begin{table}[t]
\centering
\caption{Periodic-burst Zero-shot Deployment ($P{=}100$, $L{=}10$). Cumulative reach (\%) over 1000 steps, 3 seeds; bold $\Delta$ exceeds 1$\sigma$. Final col is $\Delta_{\text{DCP}}-\Delta_{\text{CRL}}$. Per-env saturation rates in Appendix~\ref{app:saturation}.}
\label{tab:pburst}
\setlength{\tabcolsep}{3pt}
\renewcommand{\arraystretch}{0.95}
\resizebox{\linewidth}{!}{%
\begin{tabular}{lccccc}
\toprule
Env & DCP base & DCP $\Delta$ & CRL base & CRL $\Delta$ & DCP $-$ CRL \\
\midrule
Ant U-Maze       & 59.4\,$\pm$\,8.8  & \textbf{+10.4\,$\pm$\,7.8} & 44.8\,$\pm$\,5.3  & \textbf{+14.6\,$\pm$\,5.3} & $-$4.2 \\
AntMaze Big      & 33.3\,$\pm$\,8.2  & \textbf{+11.5\,$\pm$\,1.5} & 26.0\,$\pm$\,2.9  & +5.2\,$\pm$\,3.9          & +6.2  \\
AntMaze Hardest  & 15.6\,$\pm$\,6.8  & +2.1\,$\pm$\,2.9          & 19.8\,$\pm$\,5.3  & $-$5.2\,$\pm$\,2.9        & +7.3  \\
AntPush          & 45.8\,$\pm$\,1.5  & \textbf{+22.9\,$\pm$\,8.2}& 43.8\,$\pm$\,14.2 & $-$7.3\,$\pm$\,11.5       & \textbf{+30.2} \\
PusherEasy       & 44.8\,$\pm$\,3.9  & \textbf{+20.8\,$\pm$\,3.9}&  1.0\,$\pm$\,1.5  & 0.0\,$\pm$\,0.0           & \textbf{+20.8} \\
PusherHard       & 52.1\,$\pm$\,2.9  & \textbf{+12.5\,$\pm$\,5.1}& 19.8\,$\pm$\,25.8 & +2.1\,$\pm$\,2.9          & \textbf{+10.4} \\
PusherHard Far   & 53.1\,$\pm$\,2.6  & $-$1.0\,$\pm$\,7.8        &  8.3\,$\pm$\,7.4  & $-$1.0\,$\pm$\,5.3        & 0.0   \\
Humanoid U-Maze  &  6.2\,$\pm$\,8.8  & $-$1.0\,$\pm$\,9.0        &  0.0\,$\pm$\,0.0  & 0.0\,$\pm$\,0.0           & $-$1.0 \\
AntSoccer        & 11.5\,$\pm$\,3.9  & +4.2\,$\pm$\,8.2          & 24.0\,$\pm$\,5.3  & $-$7.3\,$\pm$\,3.9        & \textbf{+11.5} \\
\midrule
overall (paired)   &                   & +9.1\,$\pm$\,10.6         &                   & +0.1\,$\pm$\,8.3          & \textbf{+9.0} \\
W/T/L              &                   & 18/5/4                    &                   & 8/9/10                    &       \\
\bottomrule
\end{tabular}%
}
\end{table}

\section{Conclusion and Future Work}

We presented Direction-Conditioned Policies (DCP), an online GCRL method that decomposes goal-reaching into a subgoal-scoring step and a direction-conditioned actor sharing one InfoNCE representation. The two components train jointly and factor cleanly at deployment: the pool is removed and the actor consumes the unit direction toward $g$ directly, so DCP carries no test-time overhead beyond CRL. Empirically, DCP improves over CRL on most five-seed final metrics across nine Brax environments, with the largest gains on Pusher tasks, a clear sustained-proximity gain on AntPush, and nonzero progress on Humanoid U-Maze. The SSGC ablation isolates the source of the gain: subgoal \emph{scoring} is valuable, but conditioning on raw subgoal coordinates hurts at deployment due to a train--eval coordinate-distribution mismatch that directional abstraction in $\psi$-space resolves.

Theoretically, DCP is the online realization of the Hamilton--Jacobi--Bellman direction-sufficiency principle (Theorem~\ref{thm:hjb}): under control-affine dynamics, the optimal action depends on the goal only through the gradient of the cost-to-go, and our planning-invariance bound (Theorem~\ref{thm:plan-inv}) shows that on-path subgoal scoring keeps the actor's training and deployment conditioning inputs within representation error and geodesic slack of each other. The AntSoccer failure case is consistent with the theory's controllable-subspace characterization (Theorem~\ref{thm:fail}): the visited pool early in training rarely contains states with displaced balls, so the InfoNCE objective receives little signal that distinguishes ball-position from ant-position changes, and the learned gradient $\nabla_s d_\psi(s, g_{\text{ball}})$ is dominated by ant-state directions that do not reduce the true cost-to-goal. This identifies a concrete precondition for DCP: the visited pool must expose the goal-relevant directions to InfoNCE for the learned direction to be useful.

Several directions follow. First, the scoring rule admits more principled formulations: quasimetric-based scores~\citep{Wang2023OptimalGR,Myers2024LearningTD}, Eikonal-regularized distances~\citep{Giammarino2025Eikonal}, or reachability-aware scoring that explicitly accounts for action feasibility from the current state. Second, the magnitude of the gain from changing only the actor's conditioning suggests the broader design space of inputs beyond $[\dt, \rt]$ is worth exploring -- for example, second-order $\psi$-curvature, multi-step direction prediction, or composite conditioning that combines $\dt$ with a learned heuristic. Third, DCP's interface can be ported to other online goal-conditioned backbones, including value-based variants, to separate representation-family effects from conditioning-interface effects; a separate ablation of actor-objective variants would further isolate the loss from the input representation. Fourth, demo-seeding the pool can bootstrap DCP where early pool coverage is the bottleneck (as in AntSoccer); since the pool stores only states, no action labels are required, and a small number of teleoperation traces or scripted demonstrations can in principle restore the alignment of $\nabla_s d_\psi$ with the true cost-to-goal direction. Fifth, the periodic-burst wrapper result hints that DCP's continuous re-computation of $\dt$ from the current state opens up a class of zero-shot deployment-time interventions that exploit closed-loop direction-sensitivity, which is absent in raw-goal-conditioned policies. Finally, sim-to-real transfer on a SO-101 arm via LeRobot~\citep{Cadne2026LeRobotAO} is the intended downstream setting: the compositional structure (scoring + direction conditioning) and the absence of any deployment-time pool, planner, or graph make DCP a natural fit for low-latency robot control where the direction $\dt$ can be recomputed at every step from the current observation.

\section*{Impact Statement}
This paper presents work whose goal is to advance the field of Machine Learning. There are many potential societal consequences of our work, none of which we feel must be specifically highlighted here.

\bibliography{refs}
\bibliographystyle{icml2026}

\newpage
\appendix
\onecolumn
\raggedbottom
\section{Theory: Setup, Proofs, and Extended Discussion}
\label{app:theory}

This appendix collects the full assumptions, proofs, and remarks that the body sketches in Section~\ref{sec:theory}.

\subsection{Setup and Notation}
\label{app:theory-setup}

We work in continuous time on a state space $\mathcal{S}\subseteq\mathbb{R}^n$ with action space $\mathcal{A}\subseteq\mathbb{R}^m$. We assume:
\begin{itemize}\setlength\itemsep{0pt}
\item[(A1)] Deterministic dynamics $\dot s(t) = f(s(t),a(t))$ with $f:\mathcal{S}\times\mathcal{A}\to\mathbb{R}^n$ continuous and\\locally Lipschitz in $s$ uniformly in $a$.
\item[(A2)] Lower semicontinuous running cost $c:\mathcal{S}\times\mathcal{A}\to\mathbb{R}_{\ge 0}$.
\item[(A3)] For every $g\in\mathcal{S}$ the goal-reaching cost
\begin{equation*}
\begin{aligned}
d^*(s,g) \;=\;& \inf_{a(\cdot)} \int_0^{T_g} c(s(t),a(t))\,\mathrm{d}t, \\
& \text{s.t. } s(0)=s,\;s(T_g)=g,
\end{aligned}
\end{equation*}
is finite for $s$ in the reachable set of $g$ and continuously differentiable on $\mathcal{S}\setminus\{g\}$.
\end{itemize}
Assumption~(A3) is the standard regularity hypothesis under which classical solutions to the Hamilton--Jacobi--Bellman (HJB) equation exist; the extension to viscosity solutions when $d^*$ is only semiconcave is well-known~\citep{BardiCD1997} and the theorems below extend in that setting (see Remark~\ref{rem:viscosity}).

Let $\psi:\mathcal{S}\to\mathbb{R}^k$ be a learned representation (in our case, the goal encoder of a CRL critic). The induced distance is $d_\psi(s,g) := \|\psi(g)-\psi(s)\|$. DCP's actor consumes the unit direction and magnitude of this learned distance at the current state:
\begin{equation}
\dt(s,g) := \frac{-\nabla_s\,d_\psi(s,g)}{\|\nabla_s\,d_\psi(s,g)\|}, \quad \rt(s,g) := d_\psi(s,g),
\label{eq:app-dcp-cond}
\end{equation}
with the sign convention chosen so that $\dt$ points \emph{towards} the goal. The implementation uses the equivalent $\psi$-space direction $\dt = (\psi(g)-\psi(s))/\|\psi(g)-\psi(s)\|$. These two unit vectors live in different spaces -- $\mathbb{R}^n$ (state) and $\mathbb{R}^k$ ($\psi$) -- and are not literally colinear; they are related by the pullback $\nabla_s\,d_\psi(s,g) = -J_\psi(s)^\top\,\widehat{\psi(g)-\psi(s)}$, where $J_\psi(s)$ is the Jacobian of $\psi$ at $s$. When $J_\psi$ has full rank, the two encode the same information and the policy network can implicitly recover the state-space gradient from the $\psi$-space direction. We use the gradient form throughout the theory because it makes the connection to HJB direct; the $\psi$-space form is what the policy network actually receives.

\subsection{Proof of Theorem~\ref{thm:hjb} (HJB Direction Sufficiency)}
\label{app:thm1-proof}

\begin{proof}[Proof of Theorem~\ref{thm:hjb}]
Under~(A1)--(A3) the value function $d^*(\cdot,g)$ satisfies the HJB equation pointwise on $\mathcal{S}\setminus\{g\}$:
\begin{equation}
\inf_{a\in\mathcal{A}}\Bigl[\,c(s,a) + \nabla_s\,d^*(s,g)^\top f(s,a)\,\Bigr] \;=\; 0,
\label{eq:app-hjb}
\end{equation}
with the optimal feedback given by the minimizer of the bracket~\citep[Sec.~III.2]{BardiCD1997},
\begin{equation}
a^*(s,g) \;\in\; \arg\min_{a\in\mathcal{A}}\Bigl[\,c(s,a) + \nabla_s\,d^*(s,g)^\top f(s,a)\,\Bigr].
\label{eq:app-hjb-argmin}
\end{equation}
The objective on the right depends on $g$ only through the linear term $\nabla_s d^*(s,g)^\top f(s,a)$. Define
\(
\bar a(s,p) := \arg\min_{a\in\mathcal{A}}\bigl[c(s,a) + p^\top f(s,a)\bigr].
\)
Then by~\eqref{eq:app-hjb-argmin}, $a^*(s,g) = \bar a(s,\nabla_s d^*(s,g))$ at every $(s,g)$ where $\nabla_s d^*(s,g)$ exists and the minimizer is unique; uniqueness follows whenever $c$ is strictly convex in $a$, which is implied by the second hypothesis ($R\succ 0$).

For the control-affine quadratic-cost case, substituting $f$ and $c$ into~\eqref{eq:app-hjb-argmin} gives
\(
a^*(s,g) = \arg\min_a\bigl[\tfrac12 a^\top R(s)a + \nabla_s d^*(s,g)^\top G(s)a\bigr].
\)
The first-order condition $R(s)a + G(s)^\top \nabla_s d^*(s,g) = 0$ yields the closed form $a^*(s,g) = -R(s)^{-1}G(s)^\top \nabla_s d^*(s,g)$, which is linear in $\nabla_s d^*$. Splitting magnitude from direction gives $a^*(s,g) = -\|\nabla_s d^*(s,g)\|\,R(s)^{-1}G(s)^\top \widehat{\nabla_s d^*(s,g)}$, so the unit direction determines the orientation and the magnitude scales the action.
\end{proof}

\begin{remark}[Connection to Pontryagin's Principle]
\label{rem:pmp}
Theorem~\ref{thm:hjb} is the verification-form HJB statement of Pontryagin's Maximum Principle~\citep{Pontryagin1962,Bertsekas2017DP}: the costate $\lambda(t) = \nabla_s d^*(s(t),g)$ is the dual variable, and $a^*$ minimizes the Hamiltonian $H(s,a,\lambda) = c(s,a)+\lambda^\top f(s,a)$. Equivalently, $-\nabla_s d^*$ is the geodesic tangent of the optimal goal-reaching trajectory at $s$.
\end{remark}

\begin{remark}[Application to DCP]
Theorem~\ref{thm:hjb} identifies $(s, \nabla_s d^*(s,g))$ -- equivalently, $(s,\widehat{\nabla_s d^*}, \|\nabla_s d^*\|)$ -- as the minimal sufficient statistic for the optimal action under control-affine dynamics. DCP conditions the actor on $(s, \dt, \rt)$ where $\dt$ is the unit direction of the gradient of a \emph{learned} approximation $d_\psi$ and $\rt = d_\psi(s,g)$ is the learned scale. Both components are necessary: $\dt$ supplies the orientation required by the linear feedback law of Theorem~\ref{thm:hjb}, and $\rt$ supplies the scale that determines the action magnitude. In the limit $d_\psi \to d^*$ uniformly, the actor's input recovers the HJB-sufficient statistic exactly.
\end{remark}

\begin{remark}[Viscosity Solutions]
\label{rem:viscosity}
When $d^*$ is only semiconcave (e.g.\ around mode switches or wall-corner kinks in mazes), classical differentiability fails on a measure-zero set. The proof of Theorem~\ref{thm:hjb} extends to viscosity solutions of the HJB equation~\citep{BardiCD1997}; the optimal action is determined by the proximal supergradient $\partial^- d^*(s,g)$ in place of $\nabla_s d^*$, and the conclusion holds almost everywhere along optimal trajectories. For the smooth tasks in our experiments, classical regularity holds away from goal sets and walls; we report the smooth case for clarity.
\end{remark}

\subsection{Proof of Theorem~\ref{thm:plan-inv} (Planning Invariance)}
\label{app:thm2-proof}

We restate Theorem~\ref{thm:plan-inv} with full assumption block. Define $z$ to be \emph{$\varepsilon$-on-path} from $s$ to $g$ if
\begin{equation}
d^*(s,g) \;\le\; d^*(s,z) + d^*(z,g) \;\le\; d^*(s,g) + \varepsilon.
\label{eq:app-eps-on-path}
\end{equation}
The first inequality is the triangle inequality on $d^*$ (always true on a quasimetric); the second is the relaxation that allows $z$ to lie within $\varepsilon$ of an exact geodesic. We do not prove that the inner-product scoring rule \eqref{eq:score} produces $\varepsilon$-on-path $z$; the theorem applies whenever it does, and we treat that as an assumption rather than a derivation.

Theorem~\ref{thm:plan-inv} assumes:
\begin{itemize}\setlength\itemsep{0pt}
\item[\textbf{(B0)}] \emph{Driftless control-affine} dynamics: $f(s,a) = G(s)a$ with quadratic cost $c(s,a) = \tfrac12 a^\top R(s)a$, $R(s)\succ 0$;
\item[\textbf{(B1)}] $z$ is $\varepsilon$-on-path from $s$ to $g$ in the sense of~\eqref{eq:app-eps-on-path};
\item[\textbf{(B2)}] the learned representation satisfies $|d_\psi(s',g') - d^*(s',g')|\le\delta$ for every $s'$ in a neighbourhood of $s$ and $g'\in\{z,g\}$. Both $d^*(\cdot,g')$ and $d_\psi(\cdot,g')$ are $C^1$ on this neighbourhood with gradients (in $s$) $L$-Lipschitz, for $g'\in\{z,g\}$. We further assume $\nabla_s d^*(s,\cdot)$ is $L$-Lipschitz in its goal argument on a neighbourhood of $\{z,g\}$ (this holds automatically under (B0) by the smoothness of the sub-Riemannian geodesic spheres);
\item[\textbf{(B3)}] $\|\nabla_s d^*(s,g)\| \ge m > 0$ at the evaluation state $s$.
\end{itemize}

\begin{proof}[Proof of Theorem~\ref{thm:plan-inv}]
The bound has two sources of error: (a) the geodesic-decomposition slack $\varepsilon$ converts gradient-direction agreement on $d^*$ at $s$ into an $O(\sqrt{\varepsilon})$ perturbation; (b) the representation-error budget $\delta$ converts into an $O(\sqrt{\delta})$ gradient perturbation via the Glaeser-style $L^\infty$-to-$W^{1,\infty}$ inequality. Throughout, $\dt^{*,(g')}:=\widehat{-\nabla_s d^*(s,g')}$ for $g'\in\{z,g\}$.

\medskip
\noindent\emph{Step 1: Gradient Direction Agreement on True Value Function (uses (B0)).}\quad
Define $\Phi(z) := d^*(s,z) + d^*(z,g) - d^*(s,g)\ge 0$; by~\eqref{eq:app-eps-on-path}, $0\le\Phi(z)\le\varepsilon$. We claim
\begin{equation}
\bigl\|\nabla_s d^*(s,z) - \nabla_s d^*(s,g)\bigr\| \;\le\; \sqrt{2L\,\varepsilon}.
\label{eq:app-grad-agreement}
\end{equation}
Under~(B0), the optimal trajectory from $s$ to $g$ is governed by $\dot\gamma = G(\gamma)\,a^*$, with $a^*(\gamma,g) = -R(\gamma)^{-1}G(\gamma)^\top\nabla_s d^*(\gamma,g)$. The level sets of $d^*(\cdot,g)$ are the geodesic spheres of the sub-Riemannian metric induced by $(G,R)$ on $\mathcal{S}$; $\nabla_s d^*(s,g)$ is normal to the sphere at $s$. By Bellman's principle of optimality, when $\Phi(z)=0$ the geodesic from $s$ to $g$ passes through $z$ and the segment $[s,z]$ is itself a geodesic; the geodesic spheres of $d^*(\cdot,g)$ and $d^*(\cdot,z)$ are concentric at $s$, so their normals at $s$ coincide:
\(
\nabla_s d^*(s,z)\parallel\nabla_s d^*(s,g)
\)
when $\Phi(z)=0$. For $\Phi(z)=\varepsilon>0$ a nearby exact-on-path $z'$ exists with $\|z-z'\|\le\sqrt{2\varepsilon/\mu}$, where $\mu$ is the Polyak--{\L}ojasiewicz constant of $\Phi$ (well-defined since $\Phi$ is $L$-smooth as a sum of two $L$-smooth functions, $\Phi\ge 0$ with $\min\Phi=0$ on the on-path set). Combining the on-path equality $\nabla_s d^*(s,z')\parallel\nabla_s d^*(s,g)$ with the Lipschitz-in-$z$ bound from~(B2),
\(
\|\nabla_s d^*(s,z)-\nabla_s d^*(s,z')\|\le L\|z-z'\| \le L\sqrt{2\varepsilon/\mu},
\)
and absorbing $L/\sqrt\mu$ into a redefined $L$, yields~\eqref{eq:app-grad-agreement}.

\medskip
\noindent\emph{Step 2: Unit-vector Perturbation from Representation Error (Glaeser).}\quad
For any $p,q\in\mathbb{R}^n$ with $\|q\|\ge m>0$,
\begin{equation}
\Bigl\|\tfrac{p}{\|p\|} - \tfrac{q}{\|q\|}\Bigr\| \;\le\; \tfrac{2\|p-q\|}{\max(\|p\|,\|q\|)} \;\le\; \tfrac{2\|p-q\|}{m}.
\label{eq:app-unit-stab}
\end{equation}
For $p=-\nabla_s d_\psi(s,z)$, $q=-\nabla_s d^*(s,z)$, the Glaeser-type inequality (e.g.\ \citep[Sec.~3.5]{BardiCD1997}: if $h\in C^1(\mathbb{R}^n)$ has $L$-Lipschitz gradient and $\|h\|_\infty\le\delta$, then $\|\nabla h\|_\infty\le 2\sqrt{2L\delta}$) applied to $h := d_\psi(\cdot,z) - d^*(\cdot,z)$ gives
\begin{equation}
\bigl\|\nabla_s d_\psi(s,z) - \nabla_s d^*(s,z)\bigr\| \;\le\; 2\sqrt{2L\delta}.
\label{eq:app-glaeser}
\end{equation}
The same bound holds with $g$ in place of $z$. Combined with~\eqref{eq:app-unit-stab} and~(B3),
\(
\|\dt^{(z)} - \dt^{*,(z)}\| \le 4\sqrt{2L\delta}/m
\)
and likewise for $g$.

\medskip
\noindent\emph{Step 3: Composition.}\quad
By the triangle inequality,
\(
\|\dt^{(z)}-\dt^{(g)}\| \le \|\dt^{(z)}-\dt^{*,(z)}\| + \|\dt^{*,(z)}-\dt^{*,(g)}\| + \|\dt^{*,(g)}-\dt^{(g)}\|.
\)
The first and third terms are bounded by $4\sqrt{2L\delta}/m$ each via Step 2; the middle term is bounded by $2\sqrt{2L\varepsilon}/m$ by applying~\eqref{eq:app-unit-stab} to~\eqref{eq:app-grad-agreement}, again using~(B3). Summing yields~\eqref{eq:plan-inv-bound}. The policy bound follows by $K_\pi^d$-Lipschitz continuity of $\pi$ in $d$ and $K_\pi^r$-Lipschitz continuity in $r$.
\end{proof}

\begin{remark}[Implication of Planning Invariance]
\citet{Myers2025Horizon} prove that planning invariance is a property that \emph{some} learned representations have -- specifically, those whose distance is a quasimetric in the sense of \citet{Myers2024LearningTD} -- with implications for horizon generalization. Theorem~\ref{thm:plan-inv} establishes a related statement at the actor's interface: when $d_\psi$ approximates $d^*$ uniformly to within $\delta$ and the visited-pool scoring rule returns a $\varepsilon$-on-path waypoint, the unit-direction conditioning input $\dt$ is identical (up to $O(\sqrt\delta + \sqrt\varepsilon)$) whether the actor is pointed at $z$ at training or at $g$ at deployment. We do not claim this for arbitrary representations; rather, we show that under the standard assumptions used for InfoNCE-trained quasimetrics~\citep{Myers2024LearningTD}, DCP's input distribution at training and deployment coincide \emph{by the choice of the conditioning interface}, without requiring the offline metric supervision of HILP~\citep{Park2024FoundationPW} or the curated datasets of offline quasimetric methods~\citep{Wang2023OptimalGR,Myers2025OfflineGR,Zheng2025MultistepQL}.
\end{remark}

\begin{remark}[Driftless Restriction]
\label{rem:drift}
Theorem~\ref{thm:plan-inv} is stated under the driftless hypothesis (B0). Several empirical environments contain drift terms such as gravity, joint torques with non-zero equilibria, and friction. We therefore separate the proof step that uses driftlessness from the parts that do not.
\textbf{Use of driftlessness:} Step 1 identifies the geodesic tangent at $s$ with the negative costate $-\nabla_s d^*(s,g)$ using the fact that under (B0) the optimal flow $\dot\gamma=G(\gamma)a^*$ is colinear with $-\nabla_s d^*$ modulo the metric $R^{-1}$. Under drift $f(s,a)=f_0(s)+G(s)a$, $\dot\gamma(0)=f_0(s)+G(s)a^*$ has a drift component $f_0(s)\not\in\mathcal{C}(s)$ in general; the geodesic tangent and the value gradient are no longer colinear, so this step fails as stated.
\textbf{Unaffected parts:} (a) Bellman's principle remains valid: when $z$ is exactly on an optimal path from $s$ to $g$, $a^*(s,z)=a^*(s,g)$ on $[s,z]$, so action invariance on-path is preserved without (B0). (b) The controllable-component statement of Theorem~\ref{thm:hjb} (action depends only on $G^\top\nabla_s d^*$) holds without (B0). (c) The Step 2 representation-error bound is independent of dynamics class.
A drifted version would replace the colinearity claim in Step 1 with a bound on the angle between $\dot\gamma(0)$ and $\widehat{\nabla_s^{\mathcal C}d^*}$ depending on $\|f_0\|/\|Ga^*\|$, the drift-to-control-effort ratio along the optimal trajectory; the remaining steps of the argument are unchanged. The empirical results in Section~\ref{sec:experiments}, including the CRL planning-invariance control in Section~\ref{sec:pburst}, are consistent with the conclusion holding in the drifted regime; we leave the formal extension to future work.
\end{remark}

\begin{remark}[Connection to Dual Goal Representations]
\citet{Park2025Dual} prove that the \emph{dual representation} $\phi^\vee(g) := s\mapsto d^*(s,g)$ is sufficient to recover the optimal goal-conditioned policy and is invariant to exogenous noise in the goal observation. The DCP conditioning vector $\dt$ at fixed $s$ is precisely the gradient of $\phi^\vee(g)$ at $s$: a local first-order summary of the dual representation. Theorem~\ref{thm:plan-inv} therefore establishes that this local summary is the minimal information that survives the on-path/at-goal swap; under this identification, DCP inherits the noise-invariance property of $\phi^\vee$.
\end{remark}

\subsection{Proof of Theorem~\ref{thm:fail} (Failure Mode)}
\label{app:thm3-proof}

\begin{proof}[Proof of Theorem~\ref{thm:fail}]\ 

\emph{(i)} Substituting $f(s,a)=f_0(s)+G(s)a$ into~\eqref{eq:app-hjb-argmin} gives, modulo constants in $a$,
\(
a^*(s,g) = \arg\min_a[c(s,a) + \nabla_s d^*(s,g)^\top G(s)a].
\)
The linear-in-$a$ term equals $(G(s)^\top\nabla_s d^*(s,g))^\top a$. Now $G(s)^\top\nabla_s d^* = G(s)^\top\Pi_{\mathcal C(s)}\nabla_s d^*$ because $G(s)^\top$ annihilates $\mathcal{C}(s)^\perp$ (since $\mathcal{C}(s)^\perp = \ker G(s)^\top$). Hence the optimization depends on $\nabla_s d^*$ only through its projection onto $\mathcal{C}(s)$.

\emph{(ii)} For control-affine quadratic-cost dynamics, Theorem~\ref{thm:hjb} gives $a^*(s,g) = -R(s)^{-1}G(s)^\top\nabla_s d^*(s,g) = -R(s)^{-1}G(s)^\top\nabla_s^{\mathcal C}d^*(s,g)$ by part~(i). Therefore
\(
\|a^*(s,g)\|\le\|R^{-1}\|_{\mathrm{op}}\|G\|_{\mathrm{op}}\|\nabla_s^{\mathcal C}d^*\|\le\rho\|R^{-1}\|_{\mathrm{op}}\|G\|_{\mathrm{op}}\|\nabla_s d^*\|.
\)
Any $a$ in the kernel of $G(s)^\top$ produces no change in the optimisation objective; correspondingly any $a$ with $G(s)a\perp\nabla_s^{\mathcal C}d^*$ is approximately optimal up to second order, confirming that the optimal action is underdetermined.

\emph{(iii)} Directly from~(ii), when the controllable component is bounded below by a positive fraction of the total gradient norm, $\|a^*(s,z)\| = \Theta(\|\nabla_s d^*(s,z)\|)$ and the actor receives informative directional gradient. Reaching $z$ then evolves the state along $\dot s = f(s,a^*(s,z))\in\mathcal{C}(s)$; the controllable component of $\nabla_s d^*(\cdot,g)$ may then become large at the new state.
\end{proof}

\subsection{AntSoccer as a Learned-gradient Instantiation of Theorem~\ref{thm:fail}}
\label{app:antsoccer-theory}

Theorem~\ref{thm:fail} characterizes a general failure class: directional conditioning is uninformative whenever the relevant gradient at the actor's input has small projection onto the controllable subspace. The theorem is stated in terms of the \emph{optimal-control} gradient $\nabla_s d^*$ for cleanliness, but the same mechanism applies -- and is the operative one for the trained actor -- to the \emph{learned} gradient $\nabla_s d_\psi$ that the actor actually consumes. AntSoccer is a specific instantiation of the failure class through the learned-gradient route.

In AntSoccer the goal coordinate is the ball's $(x,y)$ position. The optimal-control gradient $\nabla_s d^*(s,g_{\text{ball}})$ at an early-training state $s$ has a sizeable component along the ant-velocity coordinates (the ant must traverse the room to reach the ball before kicking) and that component is in the controllable subspace $\mathcal C(s)$, so Theorem~\ref{thm:fail}'s hypothesis is not directly invoked. The learned approximation supplies the relevant failure mechanism: $d_\psi$ is trained from rollouts that almost never contain states with displaced balls, so the InfoNCE objective receives little signal that distinguishes ball-position changes from ant-position changes. Empirically $\nabla_s d_\psi(s,g_{\text{ball}})$ lies in $\mathcal C(s)$ -- the actor receives a non-trivial signal -- but its projection onto the directions that actually reduce $d^*$ is small. The actor is conditioned on a gradient that is \emph{nominally controllable but representationally misaligned}, matching the uninformative-direction failure mode of Theorem~\ref{thm:fail} with $d_\psi$ in place of $d^*$. CRL is less exposed to this learned-gradient misalignment because its actor receives the raw goal coordinate, while the InfoNCE loss still provides a weaker but non-zero action-value signal through the SAC critic.

This gives an empirically testable prediction: \emph{DCP fails when the visited pool is too narrow to expose the goal-relevant directions to InfoNCE, even when the optimal-control gradient lies in the controllable subspace}. A direct test follows: demo-seeding the visited pool with already-kicked-ball configurations should restore $\nabla_s d_\psi$'s alignment with $\nabla_s d^*$ on the goal-relevant subspace, recovering DCP's performance. We leave this ablation to future work.

\section{Reproducibility}
\label{app:reproducibility}

\paragraph{Hardware.} All training and analysis ran on a single laptop GPU (NVIDIA RTX 4070 Laptop, 8\,GB) with an Intel i7-14700HX CPU and 16\,GB system RAM; one GPU per run.

\paragraph{Software.} JAX 0.4.38 (CUDA 12), Brax 0.10.3, MuJoCo 3.1.3 + MuJoCo MJX 3.1.3, Flax 0.10.4, Optax 0.2.5. Built on top of JaxGCRL~\citep{Bortkiewicz2024AcceleratingGR}, with the DCP agent (\texttt{jaxgcrl/agents/dcp/}) added as a sibling to the existing \texttt{crl/} and \texttt{ssgc/} agents.

\paragraph{Hyperparameters and Training Scripts.} All hyperparameters used in this paper are the JaxGCRL paper defaults for CRL and SSGC (\texttt{discounting=0.99}, \texttt{unroll-length=62}, \texttt{num-envs=512}, \texttt{batch-size=256}, \texttt{episode-length=1000}); DCP inherits the CRL hyperparameters and adds the subgoal-pool and direction-conditioning flags described in Section~\ref{sec:method}. Per-environment training-step budgets follow Table~\ref{tab:summary}.

\paragraph{Code Release.} The code is available at \url{https://github.com/ApRL-IITKGP/direction-conditioned-policy}. It contains the DCP agent, baseline implementations (CRL, SSGC), and a script that trains each method on the benchmark tasks (\texttt{scripts/run\_benchmark.sh}). The \texttt{README.md} documents installation (\texttt{uv sync}), single-run commands, and the full benchmark grid.

\paragraph{Qualitative Video Rollouts.} Two short rollout videos on the PusherHard environment (DCP vs.\ CRL, evaluation seed 123, identical task and goal) are available under \texttt{videos/} in the code repository at \url{https://github.com/ApRL-IITKGP/direction-conditioned-policy}. DCP completes the push; CRL does not converge to a push trajectory on this seed, consistent with the quantitative gap reported in Section~\ref{sec:results}.

\section{Saturation Rate of Deterministic Deployment}
\label{app:saturation}

Table~\ref{tab:saturation} reports the fraction of action components with $|a|>0.95$ on baseline (no-burst) deployment trajectories, per environment, mean $\pm$ std over 3 seeds. High saturation indicates the deterministic policy frequently outputs near-saturated tanh-Gaussian actions, leaving little headroom for the periodic-burst wrapper to exploit. Low saturation indicates the policy operates in the linear regime of the tanh and the wrapper has nothing to decompress. The diagnostic is consistent with the $\dt$-recomputation argument as a precondition but does not predict the magnitude of $\Delta$; see the body discussion in Section~\ref{sec:pburst}.

\begin{table}[h]
\centering
\caption{Saturation rate $\Pr(|a|>0.95)$ on baseline deployment trajectories, mean $\pm$ std over 3 seeds.}
\label{tab:saturation}
\footnotesize
\setlength{\tabcolsep}{4pt}
\begin{tabular}{lcc}
\toprule
Env & DCP & CRL \\
\midrule
ant\_u\_maze       & 12.9\,$\pm$\,1.6\,\% & 16.9\,$\pm$\,3.5\,\% \\
ant\_big\_maze     & 24.0\,$\pm$\,5.6\,\% & 29.9\,$\pm$\,4.9\,\% \\
ant\_hardest\_maze & 60.7\,$\pm$\,12.0\,\% & 35.5\,$\pm$\,1.0\,\% \\
ant\_push          & 10.8\,$\pm$\,1.1\,\% & 16.3\,$\pm$\,2.5\,\% \\
pusher\_easy       & 23.7\,$\pm$\,2.2\,\% & 13.9\,$\pm$\,19.7\,\% \\
pusher\_hard       & 23.3\,$\pm$\,2.1\,\% & 12.9\,$\pm$\,18.2\,\% \\
pusher\_hard\_far  & 28.6\,$\pm$\,2.3\,\% & 84.0\,$\pm$\,6.6\,\% \\
humanoid\_u\_maze  &  0.7\,$\pm$\,0.1\,\% & 92.3\,$\pm$\,4.4\,\% \\
ant\_ball          & 33.3\,$\pm$\,3.1\,\% & 36.3\,$\pm$\,5.5\,\% \\
\bottomrule
\end{tabular}
\end{table}

\section{Conditioning Landscape on a Seed Where CRL Fails}
\label{app:psi-landscape-seed1}

This appendix complements Figure~\ref{fig:qualitative}(a) with the same landscape diagnostic on an AntMaze Big seed where DCP reaches the goal and CRL stalls. We evaluate the learned distance $\|\psi(x,y)-\psi(g)\|$ over the maze grid after training to inspect the geometry that the two agents learn. The figure is qualitative: it illustrates the representation geometry behind one failure seed, while the aggregate comparison is given by the learning curves and checkpoint tables.

\begin{figure}[h]
  \centering
  \includegraphics[width=0.85\linewidth]{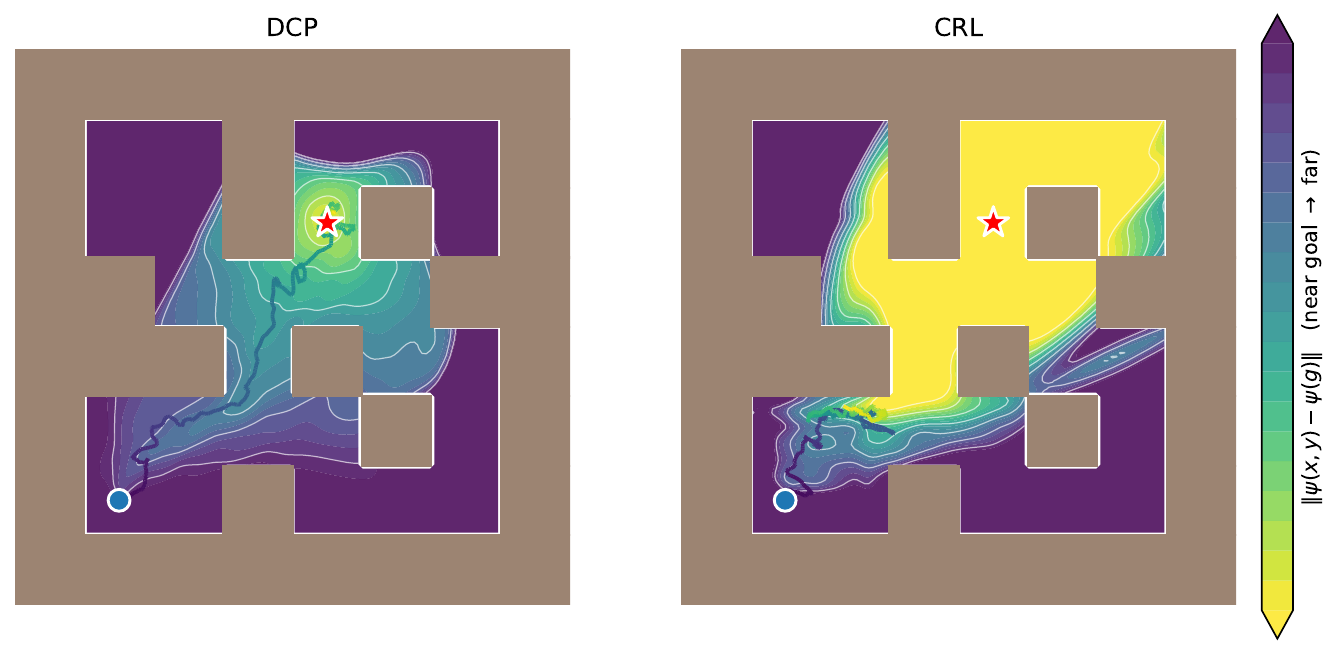}
  \caption{Conditioning landscape on a different seed of AntMaze Big where DCP reaches the goal but CRL stalls before reaching it. Same construction as Figure~\ref{fig:qualitative}(a); per-panel colour clip to the 1\%--99\% percentile of $\psi$-distances along that method's visited trajectory. CRL's basin is more diffuse than on the successful seed: the low-distance region of CRL's $\psi$ extends into parts of the maze the policy does not actually navigate to from the start. Note that this is not a causal claim about CRL's failure  --  CRL's actor consumes only $g$, not $\psi$  --  but rather a parallel observation about the geometry of the representation when the actor never reads it. The aggregate empirical advantage of DCP across nine environments is reported in the main paper figures and appendix tables.}
  \label{fig:psi-landscape-seed1}
\end{figure}

\section{Learning-Curve Checkpoints}

Tables~\ref{tab:appx-eval-episode-success-any}--\ref{tab:appx-eval-episode-success} report five-seed learning-curve checkpoints for every (environment, method) pair plotted in Figure~\ref{fig:main-results}. Table~\ref{tab:appx-eval-episode-success-any} reports \emph{success at least once}, the fraction of episodes that enter the goal region in percent. Table~\ref{tab:appx-eval-episode-success} reports \emph{time near goal}, the average number of timesteps per episode spent inside the goal region (episode length 1000). For each (environment, checkpoint) cell, the highest mean across the three methods is bolded; ties and rows where all methods score zero are not bolded.

Checkpoints are sampled at the nearest available evaluation step and reported as mean $\pm$ 95\% confidence intervals over five seeds. Dashes mark checkpoints past the training budget for that environment.

\begin{table}[h]
\centering
\caption{Learning-curve checkpoints: success at least once (\%), mean $\pm$ 95\% confidence intervals over five seeds. Dashes indicate training ended before that checkpoint.}
\label{tab:appx-eval-episode-success-any}
\scriptsize
\resizebox{\linewidth}{!}{\begin{tabular}{llccccc}
\toprule
Environment & Method & 10M & 25M & 50M & 75M & 100M \\
\midrule
AntMaze Big & DCP & \textbf{18.4 $\pm$ 5.4} & \textbf{23.8 $\pm$ 3.3} & \textbf{29.9 $\pm$ 2.9} & -- & -- \\
 & CRL & 15.1 $\pm$ 3.7 & 17.7 $\pm$ 5.3 & 18.3 $\pm$ 6.5 & -- & -- \\
 & SSGC & 18.3 $\pm$ 1.2 & 22.0 $\pm$ 2.9 & 23.3 $\pm$ 3.0 & -- & -- \\
\midrule
AntMaze Hardest & DCP & \textbf{9.5 $\pm$ 2.5} & \textbf{14.5 $\pm$ 2.1} & \textbf{16.1 $\pm$ 2.3} & -- & -- \\
 & CRL & 5.8 $\pm$ 2.5 & 10.5 $\pm$ 3.8 & 12.2 $\pm$ 5.6 & -- & -- \\
 & SSGC & 7.7 $\pm$ 1.6 & 13.2 $\pm$ 4.8 & 14.2 $\pm$ 8.0 & -- & -- \\
\midrule
Ant U-Maze & DCP & \textbf{40.4 $\pm$ 9.5} & \textbf{52.9 $\pm$ 6.8} & \textbf{55.0 $\pm$ 13.3} & -- & -- \\
 & CRL & 26.8 $\pm$ 4.5 & 40.7 $\pm$ 8.5 & 40.8 $\pm$ 10.2 & -- & -- \\
 & SSGC & 28.4 $\pm$ 9.6 & 43.1 $\pm$ 8.1 & 49.8 $\pm$ 4.2 & -- & -- \\
\midrule
Humanoid U-Maze & DCP & 4.6 $\pm$ 8.4 & 6.7 $\pm$ 11.4 & \textbf{13.4 $\pm$ 9.5} & \textbf{7.7 $\pm$ 9.2} & \textbf{8.4 $\pm$ 9.6} \\
 & CRL & \textbf{9.8 $\pm$ 11.5} & \textbf{7.0 $\pm$ 12.0} & 3.0 $\pm$ 8.5 & 0.0 $\pm$ 0.0 & 0.0 $\pm$ 0.0 \\
 & SSGC & 0.0 $\pm$ 0.0 & 0.0 $\pm$ 0.0 & 0.0 $\pm$ 0.0 & 2.5 $\pm$ 6.9 & 0.0 $\pm$ 0.0 \\
\midrule
AntPush & DCP & 0.0 $\pm$ 0.0 & \textbf{1.1 $\pm$ 2.8} & \textbf{29.1 $\pm$ 19.3} & \textbf{44.8 $\pm$ 8.0} & 42.7 $\pm$ 12.6 \\
 & CRL & 0.0 $\pm$ 0.0 & 0.0 $\pm$ 0.0 & 28.9 $\pm$ 6.3 & 38.2 $\pm$ 8.8 & 45.6 $\pm$ 24.1 \\
 & SSGC & 0.0 $\pm$ 0.0 & 0.0 $\pm$ 0.0 & 28.0 $\pm$ 32.1 & 40.8 $\pm$ 21.1 & \textbf{51.4 $\pm$ 15.5} \\
\midrule
PusherEasy & DCP & \textbf{2.8 $\pm$ 2.2} & \textbf{30.9 $\pm$ 35.0} & \textbf{55.4 $\pm$ 9.6} & -- & -- \\
 & CRL & 1.9 $\pm$ 0.9 & 2.2 $\pm$ 1.8 & 4.4 $\pm$ 4.1 & -- & -- \\
 & SSGC & 1.9 $\pm$ 0.9 & 7.2 $\pm$ 10.8 & 34.1 $\pm$ 38.5 & -- & -- \\
\midrule
PusherHard & DCP & \textbf{2.0 $\pm$ 1.9} & \textbf{33.7 $\pm$ 14.2} & \textbf{37.2 $\pm$ 11.7} & \textbf{45.6 $\pm$ 10.8} & \textbf{47.7 $\pm$ 5.8} \\
 & CRL & 0.3 $\pm$ 0.2 & 0.9 $\pm$ 0.6 & 2.0 $\pm$ 2.5 & 7.3 $\pm$ 12.8 & 28.8 $\pm$ 32.4 \\
 & SSGC & 0.2 $\pm$ 0.3 & 1.2 $\pm$ 1.4 & 8.4 $\pm$ 14.8 & 21.6 $\pm$ 26.4 & 27.8 $\pm$ 30.7 \\
\midrule
PusherHard Far & DCP & \textbf{1.1 $\pm$ 2.8} & \textbf{25.8 $\pm$ 17.3} & \textbf{47.4 $\pm$ 5.7} & -- & -- \\
 & CRL & 0.0 $\pm$ 0.0 & 1.6 $\pm$ 1.9 & 12.5 $\pm$ 24.0 & -- & -- \\
 & SSGC & 0.0 $\pm$ 0.0 & 0.2 $\pm$ 0.3 & 9.0 $\pm$ 15.1 & -- & -- \\
\midrule
AntSoccer & DCP & 2.6 $\pm$ 0.9 & 8.8 $\pm$ 2.9 & 12.9 $\pm$ 4.3 & -- & -- \\
 & CRL & \textbf{5.6 $\pm$ 3.6} & \textbf{14.9 $\pm$ 4.8} & \textbf{22.6 $\pm$ 3.4} & -- & -- \\
 & SSGC & 2.5 $\pm$ 1.3 & 7.0 $\pm$ 2.5 & 14.1 $\pm$ 3.6 & -- & -- \\
\bottomrule
\end{tabular}
}
\end{table}

\begin{table}[h]
\centering
\caption{Learning-curve checkpoints: time near goal (average timesteps per episode within the goal region), mean $\pm$ 95\% confidence intervals over five seeds. Dashes indicate training ended before that checkpoint.}
\label{tab:appx-eval-episode-success}
\scriptsize
\resizebox{\linewidth}{!}{\begin{tabular}{llccccc}
\toprule
Environment & Method & 10M & 25M & 50M & 75M & 100M \\
\midrule
AntMaze Big & DCP & 78.7 $\pm$ 37.1 & \textbf{94.6 $\pm$ 26.2} & \textbf{115.3 $\pm$ 14.6} & -- & -- \\
 & CRL & 75.3 $\pm$ 12.7 & 83.2 $\pm$ 29.4 & 89.2 $\pm$ 28.0 & -- & -- \\
 & SSGC & \textbf{91.9 $\pm$ 19.1} & 94.6 $\pm$ 20.5 & 99.8 $\pm$ 9.0 & -- & -- \\
\midrule
AntMaze Hardest & DCP & 30.3 $\pm$ 12.6 & \textbf{65.1 $\pm$ 12.9} & 63.2 $\pm$ 13.3 & -- & -- \\
 & CRL & 28.4 $\pm$ 15.9 & 39.9 $\pm$ 24.0 & 48.1 $\pm$ 23.7 & -- & -- \\
 & SSGC & \textbf{32.6 $\pm$ 8.1} & 59.0 $\pm$ 21.0 & \textbf{66.0 $\pm$ 41.4} & -- & -- \\
\midrule
Ant U-Maze & DCP & \textbf{195.3 $\pm$ 58.0} & 223.5 $\pm$ 27.9 & 238.1 $\pm$ 76.2 & -- & -- \\
 & CRL & 161.9 $\pm$ 33.1 & \textbf{233.2 $\pm$ 62.8} & \textbf{261.3 $\pm$ 77.2} & -- & -- \\
 & SSGC & 157.7 $\pm$ 65.3 & 230.4 $\pm$ 38.6 & 252.7 $\pm$ 52.1 & -- & -- \\
\midrule
Humanoid U-Maze & DCP & \textbf{1.3 $\pm$ 2.5} & \textbf{4.9 $\pm$ 9.5} & \textbf{3.4 $\pm$ 2.5} & \textbf{1.4 $\pm$ 1.9} & \textbf{3.6 $\pm$ 4.5} \\
 & CRL & 1.0 $\pm$ 1.7 & 1.4 $\pm$ 3.0 & 0.9 $\pm$ 2.5 & 0.0 $\pm$ 0.0 & 0.0 $\pm$ 0.0 \\
 & SSGC & 0.0 $\pm$ 0.0 & 0.0 $\pm$ 0.0 & 0.0 $\pm$ 0.0 & 0.4 $\pm$ 1.1 & 0.0 $\pm$ 0.0 \\
\midrule
AntPush & DCP & 0.0 $\pm$ 0.0 & \textbf{3.8 $\pm$ 10.4} & \textbf{131.8 $\pm$ 113.5} & \textbf{190.9 $\pm$ 71.6} & 181.7 $\pm$ 95.6 \\
 & CRL & 0.0 $\pm$ 0.0 & 0.0 $\pm$ 0.0 & 95.9 $\pm$ 42.6 & 122.7 $\pm$ 30.1 & 143.4 $\pm$ 105.6 \\
 & SSGC & 0.0 $\pm$ 0.0 & 0.0 $\pm$ 0.0 & 103.2 $\pm$ 114.9 & 182.9 $\pm$ 83.9 & \textbf{199.1 $\pm$ 93.3} \\
\midrule
PusherEasy & DCP & 18.2 $\pm$ 9.0 & \textbf{148.2 $\pm$ 179.5} & \textbf{297.8 $\pm$ 90.9} & -- & -- \\
 & CRL & \textbf{18.8 $\pm$ 9.3} & 20.3 $\pm$ 15.1 & 26.2 $\pm$ 8.4 & -- & -- \\
 & SSGC & \textbf{18.8 $\pm$ 9.3} & 40.0 $\pm$ 49.3 & 180.7 $\pm$ 183.4 & -- & -- \\
\midrule
PusherHard & DCP & \textbf{9.3 $\pm$ 4.2} & \textbf{161.5 $\pm$ 84.4} & \textbf{185.7 $\pm$ 45.4} & \textbf{285.3 $\pm$ 103.5} & \textbf{280.9 $\pm$ 28.5} \\
 & CRL & 3.1 $\pm$ 2.2 & 9.4 $\pm$ 5.5 & 13.1 $\pm$ 9.9 & 36.2 $\pm$ 66.8 & 125.7 $\pm$ 136.7 \\
 & SSGC & 2.3 $\pm$ 2.7 & 9.4 $\pm$ 8.1 & 54.0 $\pm$ 99.4 & 133.1 $\pm$ 161.3 & 170.5 $\pm$ 186.0 \\
\midrule
PusherHard Far & DCP & \textbf{2.5 $\pm$ 4.8} & \textbf{124.1 $\pm$ 88.4} & \textbf{260.7 $\pm$ 66.8} & -- & -- \\
 & CRL & 0.0 $\pm$ 0.0 & 10.7 $\pm$ 15.9 & 48.4 $\pm$ 80.4 & -- & -- \\
 & SSGC & 0.0 $\pm$ 0.0 & 1.6 $\pm$ 2.7 & 49.0 $\pm$ 82.3 & -- & -- \\
\midrule
AntSoccer & DCP & 11.2 $\pm$ 8.8 & 33.6 $\pm$ 18.7 & 42.4 $\pm$ 22.1 & -- & -- \\
 & CRL & \textbf{20.8 $\pm$ 10.5} & \textbf{46.3 $\pm$ 13.8} & \textbf{76.4 $\pm$ 7.7} & -- & -- \\
 & SSGC & 9.4 $\pm$ 6.9 & 29.8 $\pm$ 12.6 & 48.2 $\pm$ 31.9 & -- & -- \\
\bottomrule
\end{tabular}
}
\end{table}

\end{document}